\documentclass[]{styles/jair}

\usepackage{url}
\usepackage[noend]{algpseudocode}
\usepackage{algorithm}
\usepackage{algorithmicx}
\usepackage{multirow}
\usepackage{multicol}
\usepackage[T1]{fontenc}
\usepackage{amsthm}

\newtheorem{theorem}{Theorem}

\newcommand{\mGF}{\mathcal{GF}\xspace}

\newcommand{\rem}[1]{}
\newcommand{\aleft}{A_\mathrm{left}}
\newcommand{\alayer}{A_\mathrm{layer}}

\newcommand{\res}{\textit{res}\xspace}

\newcommand{\ENHSP}{\textsc{enhsp}\xspace}

\newcommand{\NLM}{\textsc{nlm}\xspace}
\newcommand{\FF}{\textsc{ff}\xspace}

\newcommand{\pattygkr}{\ensuremath{\textsc{Patty}_\textsc{s}}\xspace}
\newcommand{\pattyhkr}{\ensuremath{\textsc{Patty}_\textsc{d}}\xspace}

\newcommand{\pattyo}{\ensuremath{\textsc{Patty}_\textsc{o}}\xspace}

\newcommand{\re}{{\ensuremath{R^2\exists}}}

\newcommand{\pattyr}{\ensuremath{\textsc{Patty}_\textsc{r}}\xspace}

\newcommand{\pattyc}{\ensuremath{\textsc{Patty}_\textsc{c}}\xspace}
\newcommand{\pattyb}{\ensuremath{\textsc{Patty}_\textsc{b}}\xspace}
\newcommand{\pattyg}{\ensuremath{\textsc{Patty}_\textsc{g}}\xspace}

\newcommand{\pattyd}{\ensuremath{\textsc{Patty}_\textsc{d}}\xspace}

\newcommand{\stp}{\textsc{computeS2P}}
\newcommand{\ptg}{\textsc{computeP2G}}
\newcommand{\stpn}{\textsc{computeS2Pn}}
\newcommand{\ptgn}{\textsc{computeP2Gn}}
\usepackage{xspace}
\usepackage{comment}
\usepackage{bm}
\usepackage{amsthm}
\usepackage{todonotes}

\newcommand{\set}[1]{\ensuremath{\{#1\}}\xspace}
\newcommand{\tuple}[1]{\ensuremath{\langle #1 \rangle}\xspace}

\newcommand{\op}[1]{\mathrm{#1}}
\newcommand{\lb}{\underline}
\newcommand{\ub}{\overline}

\newcommand{\pluseq}{\mathrel{+}=}
\newcommand{\minuseq}{\mathrel{-}=}

\newcommand{\mA}{\mathcal{A}\xspace}

\newcommand{\mG}{\mathcal{G}\xspace}

\newcommand{\mI}{\mathcal{I}\xspace}

\newcommand{\mS}{\mathcal{S}\xspace}
\newcommand{\mT}{\mathcal{T}\xspace}

\newcommand{\mX}{\mathcal{X}\xspace}

\newcommand{\rational}{\ensuremath{\mathbb{Q}}\xspace}

\renewcommand{\natural}{\ensuremath{\mathbb{N}}\xspace}

\newcommand{\real}{\ensuremath{\mathbb{R}}\xspace}

\usepackage{pifont}

\newcommand{\pre}{\op{pre}}

\newcommand{\eff}{\op{eff}}

\renewcommand{\implies}{\rightarrow}

\newcommand{\liff}{\leftrightarrow}

\newcommand{\patty}{\textsc{patty}\xspace}

\newcommand{\pattern}{{\prec\xspace}}

\usepackage{ stmaryrd }

\newcommand{\asseq}{:=}

\newcommand{\smt}{\textsc{smt}\xspace}

\newcommand{\spp}{\textsc{spp}\xspace}

\newcommand{\ipc}{\textsc{ipc}\xspace}
\newcommand{\ite}{\textsc{ite}\xspace}

\DeclareFontFamily{U}{mathb}{\hyphenchar\font45}
\DeclareFontShape{U}{mathb}{m}{n}{
<-6> mathb5 <6-7> mathb6 <7-8> mathb7
<8-9> mathb8 <9-10> mathb9
<10-12> mathb10 <12-> mathb12
}{}
\DeclareSymbolFont{mathb}{U}{mathb}{m}{n}
\DeclareMathSymbol{\cespatternbasic}{\mathrel}{mathb}{"CE}

\newcommand{\arpg}{\textsc{arpg}\xspace}

\newcommand{\ENHSPCT}{{\textsc{ENHSP}}$_{\textsc{CT}}$}

\setcopyright{cc}
\copyrightyear{2025}
\acmDOI{10.1613/jair.1.xxxxx}

\JAIRAE{TBD}
\JAIRTrack{} 
\acmVolume{0}
\acmArticle{0}
\acmMonth{0}
\acmYear{0000}

\RequirePackage[
  datamodel=acmdatamodel,
  style=acmauthoryear,
  backend=biber,
  giveninits=true,
  uniquename=init
  ]{biblatex}

\begin{document}

\title[Exploiting Search in Symbolic Numeric Planning with Patterns]{Exploiting Search in Symbolic Numeric Planning with Patterns}

\author{Matteo Cardellini}
\authornote{Corresponding Author.}
\orcid{0000-0003-3788-9475}
\email{matteo.cardellini@unige.it}
\affiliation{%
  \institution{DIBRIS, University of Genoa}
  \city{Genova}
  \country{Italy}
}

\author{Enrico Giunchiglia}
\orcid{0000-0001-5758-2556}
\email{enrico.giunchiglia@unige.it}
\affiliation{%
  \institution{DIBRIS, University of Genoa}
  \city{Genova}
  \country{Italy}
}

\renewcommand{\shortauthors}{Cardellini \& Giunchiglia}

\begin{abstract}
In this paper, we present a procedure for numeric planning based on Symbolic Pattern Planning (\spp). Given a numeric planning problem $\Pi$,  a pattern $\pattern$ is a sequence of actions used to define a formula encoding the subsequences of $\pattern$ executable from a starting state $S$. 
\citet{DBLP:conf/aaai/CardelliniGM24} follow the Planning as Satisfiability approach by defining, at each step $n \ge 0$, a formula $\Pi^\pattern_n$ in which $(i)$ the pattern $\pattern$ is computed only for $n=0$ in the initial state $I$ of $\Pi$, and then exploited at each step $n$, $(ii)$ the starting state $S$ is set to $I$, and $(iii)$ the set $G$ of goals is required to hold in the last state that can be reached by one of the subsequences of $\pattern$ concatenated $n$ times. The procedure begins with $n=0$, terminates as soon as $\Pi^\pattern_n$ is satisfiable, and otherwise proceeds by incrementing $n$.
In this paper, possibly at each step, $(i)$ we symbolically search for an intermediate state $P$ reachable from $I$, closer to a goal state, $(ii)$ dynamically recompute the pattern $\pattern_h$ --to be used in the next step-- in $P$, $(iii)$ refine the pattern $\pattern_g$ used to reach $P$, and $(iv)$ start the new search from the state $S$ which can be either the initial state $I$ or the last computed intermediate state $P$, exploiting the computed patterns $\pattern_g$ and $\pattern_h$ to define the pattern $\pattern$ to be used in the search. In particular, at each step, we define a formula $\Pi^{\pattern}_{S,P}$ encoding the existence of a state $P'$ closer than $P$ to a goal state, with $P'$ reachable from the starting state $S$ when using the pattern $\pattern$. 
We present different techniques for producing such formulas, each corresponding to a different strategy for exploring the search space. We prove their correctness and completeness, the latter under certain conditions. We evaluate their performance against both the original~\spp approach and other publicly available numeric planners, using the settings and benchmarks from the 2023 International Planning Competition Agile track. 
The results highlight that the proposed procedures achieve good overall performance compared to state-of-the-art numeric planners, including the original \spp{} approach. Ablation studies highlight which technique is better suited to which domain.
\end{abstract}


\maketitle

\section{Introduction}
In a deterministic \textsc{ai} planning problem $\Pi$, the objective is to find a sequence in a set $A$ of possible actions leading from an initial state $I$ to a state  satisfying a set~$G$ of goals.
Two are the main alternative approaches  to solve a planning problem. \textsl{Planning as Search}, given a set of states already determined to be reachable from the initial state, selects one of these states and progress it with a sequence of actions likely to lead to a goal state,
see, e.g., 
\cite{Bonet_Geffner_2001}.
\textsl{Planning as Satisfiability}
$(i)$ fixes a \textsl{bound} or \textsl{number of steps} $n$, initially set to $0$,  $(ii)$ computes a logic formula encoding the sequences of actions of length $n$ which are executable in a given starting state, $(iii)$ checks whether one of such sequences is a valid plan by imposing $I$ and $G$ as conditions on the starting and the final state, and $(iv)$ upon failure, iterates from the second step while increasing $n$, see, e.g. \cite{DBLP:conf/kr/KautzMS96}. 
Symbolic Pattern Planning (\spp) \cite{DBLP:conf/aaai/CardelliniGM24} is a recent logic encoding of numeric planning problems, which, given an arbitrary sequence $\pattern$ of actions called {\sl pattern}, allows producing a formula encoding all the executable subsequences of the pattern $\pattern$.
\citet{DBLP:conf/aaai/CardelliniGM24} compute a static pattern $\pattern$ once in the initial state, and then exploit it at each iteration in the Planning as Satisfiability setting. The proposal showed to be very competitive with state-of-the-art (\textsc{sota}) planning systems when using the settings and benchmarks from the 2023 International Planning Competition (\ipc) Agile track \cite{ipc2023}. 
 
In this paper, we push the envelope of numeric planning with patterns, showing how to symbolically search for a valid plan  exploiting patterns.
Specifically, given a state $P$ reachable from the initial state $I$, at each iteration we define a formula $\Pi^{\pattern}_{S,P}$ enforcing that:
\begin{enumerate}
    \item the search starts from a selected state $S$;
    \item the pattern $\pattern$ is obtained by concatenating:
    \begin{enumerate}
        \item a prefix pattern $\pattern_g$ that allows reaching the state $P$ from $S$, and
        \item a suffix pattern $\pattern_h$, computed in $P$, aimed at progressing the search beyond $P$;
    \end{enumerate}
    \item there exists a state $P'$ that is closer than $P$ to a goal state, according to a goal value function $GF_P$, which assigns a positive value to every non-goal state.
\end{enumerate}
At the beginning, the state $P$ and $S$ are both set to the initial state $I$, $\pattern_g$ is empty and $\pattern_h$  is computed from the initial state. Then, assuming such a new state $P'$ is determined, the procedure terminates if $P'$ is a goal state, and, otherwise, $(i)$ the state $P$ is set to $P'$, $(ii)$
the starting state $S$ is set to  either $P$ or to 
the initial state $I$, and $(iii)$ $\pattern_g$, $\pattern_h$ and $GF_P$ --and thus the new formula $\Pi^\pattern_{S,P}$ to be used in the next iteration-- are recomputed. When the search for such a state $P'$ fails, the procedure is iterated, possibly setting a new starting state $S$, $\pattern_g$ and $\pattern_h$. 
Depending on the function $GF_P$ used to decide when a state $P'$ is closer than $P$ to a goal state, and  on how the state $S$ and the patterns $\pattern_g$ and $\pattern_h$ are defined at each iteration, we obtain different $\Pi^\pattern_{S,P}$ formulas, each leading to a different exploration of the search space. The {\spp} approach by \citet{DBLP:conf/aaai/CardelliniGM24} can be seen as the special case of the above outlined procedure in which in each iteration no intermediate state $P'$ is ever computed (and thus the
starting state $S$ is always fixed to the initial state $I$), $\pattern_g$ is empty and the pattern $\pattern$ used at each iteration is obtained by concatenating once more the initially computed pattern $\pattern_h$. In our case,  similarly (resp. differently) to planning as search (resp. standard planning as satisfiability), at the $n$-th iteration,
$(i)$ the starting state $S$ of the search can be an intermediate state $P$ which is reachable from $I$, $(ii)$ when $S$ is set to $I$, the pattern $\pattern_g$ allowing reaching $P$ starting from $S$ does not necessarily encode all executable plans of length $n$; and
$(iii)$ the pattern $\pattern_h$ is selected dynamically during the search and can contain a subset of the available actions.

On the theoretical side, we prove the correctness of the general procedure (any returned plan is valid) and its completeness  (if there exists a valid plan, one will be returned), the latter under certain conditions. We then focus on four specific exploration strategies, that we dub as \emph{cautious}, \emph{brave}, \emph{reckless} and \emph{greedy}. We experimentally evaluate them using the settings of the 2023 \ipc Agile track, i.e., 20 domains and 20 problems per domain. To these, we added the 20 problems of the Line Exchange domain introduced by \citet{DBLP:conf/aaai/CardelliniGM24}. 
The analysis indicates that the proposed four specific procedures have comparatively better performance on different domains, and overall good performance compared to both the previous \spp approach and all the other publicly available planners, both symbolic and search-based.
Overall, the cautious/brave/reckless/greedy procedures solve respectively 326/330/316/315 problems out of the 420 considered, compared to the 291 
solved by the original {\spp} approach proposed by \citet{DBLP:conf/aaai/CardelliniGM24}, and to the 276 solved by the best performing \textsc{sota} planner not based on patterns.
Ablation studies confirm that the specific combination of techniques used in the cautious, brave, reckless and greedy procedures lead to the best overall performance, but that in some domain the exploitation of some alternative technique leads to even better performance.

This makes use of the \spp approach presented by \citet{DBLP:conf/aaai/CardelliniGM24}, which itself builds upon $(i)$ the \re-encoding for classical planning introduced by \citet{balyo_relaxing_2013} and adapted to numeric planning by \citet{bofill_espasa_villaret_2016}, and $(ii)$ the rolling encoding described by \citet{Scala_Ramirez_Haslum_Thiebaux_2016_Rolling}. In search-based planning, several planners leverage action sequences by first attempting to reach a goal state  sequences of actions, reverting to individual actions if unsuccessful. The classical planner \textsc{yahsp} \cite{vidal2004yahsp} utilizes “look-ahead plans” (i.e., sequences) during forward search to leap to intermediate states closer to the goal, an approach conceptually similar to macro-actions \cite{alarnaouti2024macroactions}. Patterns, however, can capture a broader superset of sequences than macro-actions or “look-ahead plans”, as actions at any position within a pattern may be omitted from the plan or rolled \cite{Scala_Ramirez_Haslum_Thiebaux_2016_Rolling}, whereas the other approaches always rely on the full sequence. More important, all of the approaches mentioned above have been implemented exclusively in search-based planning, while our work pushes forward the \textsc{sota} in satisfiability-based planning. Finally, a preliminary version of this paper appeared in \citet{DBLP:conf/kr/CardelliniG25} where we presented the procedure that we now call \emph{cautious}.

Summing up, the main contributions of the paper are:
\begin{enumerate}
    \item We present a general  \spp procedure for numeric planning with dynamic recomputation of the pattern.
    \item We prove its correctness and present conditions which ensure also its completeness.
    \item We specialize the general procedure by considering the \emph{cautious}, \emph{brave}, \emph{reckless} and \emph{greedy} strategies for exploring the search space.
    \item We show that these procedures exhibit varying performance across different domains, and that they perform competitively with publicly available planners on the setting and benchmarks used in the 2023 \ipc, Agile track.
    \item We conduct ablation studies, revealing that the four procedures have the best overall performance, though some alternative strategy leads to better performance on some domain.
\end{enumerate}

After the background on numeric Planning as Satisfiability with \spp (Section~\ref{sec:background}), we present our general procedure in Section~\ref{sec:pattyd}, starting with a simple example used to illustrate the drawbacks of the original \spp approach and the intuitions behind it. In Section~\ref{sec:cautious-brave-reckless-greedy}, we introduce the cautious, brave, reckless, and greedy variants of the general procedure, prove their formal properties, and discuss their differing behaviors on the motivating example. The implementation and the experimental analysis are presented in Section~\ref{sec:impl-exper}, which is followed by the ablation studies in Section~\ref{sec:ablation}. We end the paper with the conclusions. 

\section{Numeric Planning as Satisfiability with Patterns} \label{sec:background}

In this section, we first introduce the syntax and semantics of \textsc{pddl2.1}  level~2 \cite{Fox_Long_2003}, the de-facto standard in numeric planning (Subsection~\ref{subsec:pddl}). Then, in Subsection~\ref{subsec:planning-as-sat-patterns} we show how  \textsc{pddl2.1}  problems have been encoded in Satisfiability Modulo Theories (\smt) \cite{BarFT-SMTLIB} exploiting the Planning as Satisfiability approach and the patten encoding. In the last Subsection~\ref{subsec:arpg}, we detail the procedure used for the pattern computation.

\subsection{Numeric Planning} 
\label{subsec:pddl}

In \textsc{pddl2.1}, a {\sl numeric planning problem} $\Pi$ is a tuple $\tuple{X, A, I, G}$ representing the sets of {\sl state} and {\sl action variables}, the {\sl initial} and {\sl goal states} respectively. The set $X$ of {\sl state variables} is the union of the sets 
$V_B$ and $V_N$ of propositional and numeric variables with domain in $\set{\top, \bot}$ and $\rational$, respectively, where $\top$ and $\bot$ are the symbols for truth and falsity. A {\sl condition} is either a propositional or a numeric condition. A {\sl propositional condition} is an expression of the form $v = \top$ or $v = \bot$, with $v \in V_B$. A {\sl numeric condition} has the form $\psi \unrhd 0$, with $\unrhd \in \set{\ge, >}$ and $\psi$ an arbitrary expression in $V_N$. For simplicity, and following common practice, we restrict to the case where $\psi$ is a linear combination of the variables in $V_N$. i.e., $\psi = \sum_{x \in V_N} k_x x + k$, with $k_x, k \in \rational$. A {\sl state} is a function assigning each variable in $X$ to an element in its domain, and is naturally extended to expressions in $V_N$, conditions  and formulas, the latter defined as propositional combination of conditions. A state $S$ {\sl satisfies a condition} $v = \top$ (resp. $w = \bot$, resp. $\psi \unrhd 0$) if $S(v) = \top$, (resp. $S(w) = \bot$, resp. $S(\psi) \unrhd 0$) and {\sl a formula} according to the truth tables of propositional logic.
In $\Pi$,
$I$ is the {\sl initial state} and $G$ is a finite set of disjunctions of conditions, each disjunct in $G$ called {\sl goal}. The models of $G$ are the {\sl goal states}. An {\sl action} $a \in A$ is a pair $\tuple{\pre(a), \eff(a)}$ where $\pre(a)$ is a set of conditions called the {\sl preconditions} of $a$, and $\eff(a)$ is the set of {\sl effects} of $a$, i.e., expressions of the form $v \asseq \top$ or $w \asseq \bot$ or $x \asseq \psi$ with $v,w \in V_B$, $x \in V_N$  and $\psi$ a linear expression. For each action $a$, the variables occurring in $\eff(a)$  to the left of the ``$\asseq$" symbol are said to be {\sl assigned by $a$}.
A numeric effect $x := \psi$ is said to be a {\sl linear increment} if $\psi = x + \psi'$ with $\psi'$ a linear expression not containing $x$. As standard, we write $(\psi \unrhd \psi')$ for $(\psi - \psi' \unrhd 0)$ and we abbreviate $x := x + \psi$ and $x := x - \psi$ with $x \pluseq \psi$ and $x \minuseq \psi$ respectively, $v = \top$ and $v = \bot$ with $v$ and $\neg v$, respectively.
From here on, $v, w, x$ represent variables and $\psi$ a numeric expression, each symbol possibly decorated with subscripts or superscripts.

An action $a$ is \textsl{executable} in a state $S$ if $S$ satisfies all the preconditions of $a$. Given an action $a$ executable in a state $S$, the {\sl result of executing} $a$ in $S$ is the state $S' = \res(a,S)$ such that, for each $v \in X$, $S'(v) = \top$ if $v := \top \in \eff(a)$, $S'(v) = \bot$ if $v := \bot \in \eff(a)$, $S'(v) = S(\psi)$ if $v := \psi \in \eff(a)$, and $S'(v) = S(v)$ otherwise. 

Consider a {\sl plan} $\pi$, defined as a finite sequence of actions $a_1;a_2;\ldots;a_{n}$ of {\sl length} $n \geq 0$. The \textsl{state sequence $S_0; S_1; \ldots; S_n$ induced by $\pi$ in $S_0$} is such that for $i \in \set{1,2,\dots, n}$, $S_{i}$ $(i)$ is undefined if either $a_{i}$ is not executable in $S_{i-1}$ or $S_{i-1}$ is undefined, and $(ii)$ is the result of executing $a_{i}$ in $S_{i-1}$ otherwise. We say that $\pi$ is \textsl{executable in a state $S$} if each state in the sequence induced by $\pi$ in $S$ is defined, and in such a case, the {\sl result of executing $\pi$ in $S$} is defined as
$\res(\pi,S) = \res(a_n,\ldots,\res(a_2,\res(a_1,S))\ldots)$.
If $\pi$ is executable in the initial state $I$ and the resulting state $\res(\pi,I)$ is a goal state, we say that $\pi$ is a {\sl valid plan}. 

\subsection{Numeric Planning as Satisfiability with Patterns}
\label{subsec:planning-as-sat-patterns}

Let $\Pi = \langle X, A, I,G\rangle$ be a numeric planning problem. In Planning as Satisfiability \cite{DBLP:conf/ecai/KautzS92,DBLP:conf/aaai/KautzS96}, an \textsl{encoding $E$ of $\Pi$} is a tuple $\langle \mX, \mA, \mI(\mX), \mT(\mX,\mA,\mX'), \mG(\mX)\rangle$ where $\mX$ is a finite set of propositional and numeric \textsl{state variables} including $X$; $\mA$ is a finite set of \textsl{ action variables}, each one with the set of values it can take;  $\mI(\mX)$ is the \textsl{initial state formula} in the set $\mX$ of variables, defined as
$$ 
\bigwedge_{v: I(v) = \top} v  \wedge \bigwedge_{w: I(w) = \bot} \neg w  \wedge \bigwedge_{x, k: I(x) = k} x = k;
$$
while $\mG(\mX)$ is the \textsl{goal formula} in the set $\mX$ of variables, obtained by making the conjunction of the goals in $G$, once each $v = \top$ and $v = \bot$ are substituted with $v$ and $\neg v$, respectively. 
The valid transitions between states correspond to the models of $\mT(\mX,\mA,\mX')$, the
\textsl{symbolic transition relation}, a formula in the variables $\mX \cup \mA \cup \mX'$, where 
 $\mX'$ is the set of {\sl next state variables} consisting of a new variable $v'$ for each variable $v \in\mX$. 
 
 In \spp, the definition of the symbolic transition relation starts by fixing a {\sl pattern (of length $k \ge 0$)}, defined as a finite sequence $a_1; a_2; \ldots; a_k$ of actions in $A$. The empty pattern, obtained for  $k=0$, is  denoted with $\epsilon$. Consider a pattern $\pattern = a_1; a_2; \ldots; a_k$, $k \ge 0$. The basic idea of \spp is to define the value of each state variable in the state resulting from the execution of each action in $\pattern$ for $0$ or more times, as a function of both the state in which the execution is started and of the pattern $\pattern$. Notice that, by definition, the pattern can contain multiple, even consecutive, occurrences of the same action $a$. However, each repeated occurrence is treated as a different copy of the action. With such assumption, with $a_i$ we denote both the $i$-th action in the pattern and the corresponding action variable in the encoding. Then, in the {\sl pattern  $\pattern$-encoding of $\Pi$}, 
\begin{enumerate}
    \item The set $\mX$ of state variables  is $X$, i.e., the set of propositional and numeric variables of $\Pi$,
    \item The set $\mA^\pattern$ of action variables is defined as $\set{a_1, a_2, \ldots, a_k}$, with one variable per action occurrence in the pattern~$\pattern$. Each variable $a_i$ ranges over the non-negative integers. Intuitively, the value of $a_i$ represents the number of times that the action is executed consecutively, following the sequential execution of each action in $\set{a_1,a_2,\ldots,a_{i-1}}$, each also executed zero or more times. Thus, in \spp, depending on the fixed pattern $\pattern$, we have a different set  $\mA^\pattern$ of action variables. 
\end{enumerate}

The value taken by $v \in X$ after the sequential execution  of each action occurrence $a_i$ in~$\pattern$ for a number $\ge~0$ of consecutive times, is given by $\sigma_i(v)$, inductively defined as $\sigma_0(v) = v$, and for each $i \in\set{1,2,\ldots,k}$
\begin{enumerate}
    \item $\sigma_i(v) = \sigma_{i-1}(v)$ if $v$ is not assigned by $a_i$,
    \item $\sigma_i(v) = (\sigma_{i-1}(v) \vee a_i > 0)$ if $v := \top \in \eff(a_i)$,
    \item $\sigma_i(v) = (\sigma_{i-1}(v) \wedge a_i = 0)$ if $v := \bot \in \eff(a_i)$,
    \item $\sigma_i(v) = (\sigma_{i-1}(v) + a_i \times \sigma_{i-1}(\psi))$ if $v \pluseq \psi \in \eff(a_i)$ is a linear increment, where $\sigma_{i-1}(\psi)$ is the expression obtained from $\psi$ after each variable $x \in V_N$ has been replaced with $\sigma_{i-1}(x)$,
    \item $\sigma_i(v) = \ite(a_i > 0, \sigma_{i-1}(\psi), \sigma_{i-1}(v))$ if $v := \psi \in \eff(a_i)$ is not a linear increment.
\end{enumerate}
The term $\ite(c,t,e)$ for ``\textit{If (c) Then $t$ Else $e$}"
returns $t$ or $e$ depending on whether the condition $c$ is true or not, and is part of the standard language supported by \smt solvers.
The symbolic transition relation of the $\pattern$-encoding --which defines the value of the variables in $\mX'$ on the basis of the values of the  variables in $\mX$ and in $\mA^\pattern$-- is denoted with $\mT^\pattern(\mX, \mA^\pattern, \mX')$, and is defined as the conjunction of the formulas in the sets:
\begin{enumerate}
\item $\op{pre}^\pattern(A)$, which contains, for each action $a_i$ in $\pattern$ and for each $v = \top$ and $w = \bot$ in $\op{pre}(a_i)$,
    $$
    a_i > 0 \implies \sigma_{i-1}(v), \qquad a_i > 0 \implies \neg \sigma_{i-1}(w),
    $$
    and for each $\psi \unrhd 0$ in $\op{pre}(a_i)$,
    $$
        a_i > 0 \implies \sigma_{i-1}(\psi) \unrhd 0, \qquad
    a_i > 1 \implies \sigma_{i-1}(\psi[a_i]) \unrhd 0,
    $$
where $\psi[a_i]$ is the linear expression obtained from $\psi$ by substituting each variable $x \in V_N$ with 
    \begin{enumerate}
        \item $x + (a_i-1) \times \psi'$, whenever $x \pluseq \psi'\in \eff(a_i)$ is a linear increment,
        \item $\psi'$, if $x := \psi' \in \op{eff}(a_i)$ is not a linear increment,
        \item $x$, if $x$ is not assigned by $a_i$.
    \end{enumerate} 
    The above formulas ensure that the numeric precondition 
    $\psi \unrhd 0 \in \op{pre}(a_i)$  holds both in the first and in the last state in which $a_i$ is executed, and thus that  $\psi \unrhd 0$  holds also in all the intermediate states in which $a_i$ is consecutively executed, see \cite{Scala_Ramirez_Haslum_Thiebaux_2016_Rolling}.


\item 
    $\op{amo}^\pattern(A)$, which contains, for each action $a_i$ in $\pattern$  which is not eligible for rolling \cite{Scala_Ramirez_Haslum_Thiebaux_2016_Rolling} 
    \begin{equation*}
        a_i = 0 \vee a_i = 1.
    \end{equation*}   
    An action $a_i$ is {\sl eligible for rolling} if 
    \begin{enumerate}
        \item $v = \bot\in \op{pre}(a_i)$ (resp. $v = \top\in \op{pre}(a_i)$) implies $v \asseq \top\not\in \op{eff}(a_i)$ (resp. $v \asseq \bot\not\in \op{eff}(a_i)$), and
        \item all the numeric variables assigned by $a_i$ with a linear increment do not occur elsewhere in $\op{eff}(a_i)$, and 
        \item $a_i$ contains a linear increment.
    \end{enumerate} 
\item $\op{frame}^\pattern(X)$, consisting of, for each variable $v \in V_B$ and $x \in V_N$, 
\begin{equation*}\begin{array}{c}
v' \liff \sigma_k(v), \qquad x' = \sigma_k(x).
\end{array} \end{equation*}     
\end{enumerate}

\rem{
For the {\sl correctness} of the encoding, each model $\mu$ of $\mT(\mX,\mA,\mX')$
has to correspond to at least one sequence of actions $\alpha$ such that $(i)$ $\alpha$ is executable in the state  $S$ such that, for each variable $v \in X$, $S(v) = \mu(v)$; and $(ii)$ the last state induced by $\alpha$ executed in $S$ is the state $s'$ such that, for each variable $v \in X$, $s'(v) = \mu(v')$. 
For the {\sl completeness} of the encoding, for each state $S$ and action $a$ executable in $S$, if $s'$ is the state resulting from the execution of $a$ in $S$, then there must be a model $\mu$ of $\mT(\mX,\mA,\mX')$ such that for each variable $v \in X$, $S(v) = \mu(v)$ and $s'(v) = \mu(v')$.}

Given $\mT^\pattern(\mX, \mA^\pattern, \mX')$,
following the Planning as Satisfiability approach, 
\begin{enumerate}
    \item an integer $n \geq 0$ called {\sl bound} or {\sl number of steps} is fixed, 
    \item $n+1$ pairwise disjoint copies  $\mX_0,\mX_1,\ldots,\mX_n$ of the set  $\mX$ of state variables, and $n$ pairwise disjoint copies $\mA^\pattern_1,\mA^\pattern_2,\ldots,\mA^\pattern_{n}$ of the set $\mA^\pattern$ of action variables are made, and 
    \item the \textsl{$\pattern$-encoding of $\Pi$ with bound $n$} is defined to be the formula 
\begin{equation}\label{eq:enc-sat}
     \Pi^\pattern_n = \mI(\mX_0) \wedge \bigwedge_{i=1}^{n} \mT^\pattern(\mX_{i-1},\mA^\pattern_i,\mX_{i}) \wedge \mG(\mX_n),
\end{equation}
in which  $\mI(\mX_0)$ is the formula in the variables $\mX_0$ obtained by substituting each  variable $x \in \mX$ with $x_0 \in \mX_0$ in $\mI(\mX)$, and similarly for  $\mT^\pattern(\mX_{i-1},\mA^\pattern_i,\mX_{i})$ and $\mG(\mX_n)$. 
\end{enumerate}
Then, the satisfiability of $\Pi^\pattern_n$ is checked calling an {\smt} solver starting from $n=0$ and then incrementing $n$ until a model is found. 

\citet{DBLP:conf/aaai/CardelliniGM24} showed that for any pattern~$\pattern$ and bound $n$, any model of $\Pi^\pattern_n$ corresponds to a valid plan of $\Pi$ ({\sl correctness}), and that if $\Pi$ has a valid plan then for any complete pattern~$\pattern$ there exists a bound $n$ for which $\Pi^\pattern_n$ is satisfiable ({\sl completeness}).
A pattern is {\sl complete} (resp. {\sl elementary}) if it contains at least (resp. at most) one occurrence of each action in the set $A$ of actions of $\Pi$. 

\rem{
\begin{theorem}[\citeauthor{DBLP:conf/aaai/CardelliniGM24} \citeyear{DBLP:conf/aaai/CardelliniGM24}]\label{th:corr-compl-pattern}
	Let $\Pi$ be a numeric planning problem and $\pattern$ be a pattern. The pattern $\pattern$-encoding  is correct. If $\pattern$ is complete, then the pattern $\pattern$-encoding is complete.
\end{theorem}

The above theorem holds for any pattern
$\pattern$, not necessarily elementary or complete. 
When 
$\pattern$ is elementary and complete, \citeauthor{DBLP:conf/aaai/CardelliniGM24} also proved that the $\pattern$-encoding of $\Pi$
allows finding a valid plan with a bound $n$ that is possibly lower and never higher than the bound necessary when using $(i)$ the standard encoding based on effect and explanatory frame axioms (see, e.g., \cite{Leofante2020}), or $(ii)$ the \re{} encoding by \citet{DBLP:conf/ijcai/BofillEV17} when using $\pattern$ as total ordering on actions, or $(iii)$ the rolling encoding by \citet{Scala_Ramirez_Haslum_Thiebaux_2016_Rolling}.}

\subsection{Computing Patterns with Asymptotic Relaxed Planning Graphs}\label{subsec:arpg}

Let $\Pi = \langle X, A, I,G\rangle$ be a numeric planning problem. For selecting the pattern, \citeauthor{DBLP:conf/aaai/CardelliniGM24}~\citeyear{DBLP:conf/aaai/CardelliniGM24} exploited the 
 Asymptotic Relaxed Planning Graph (\arpg)  construction from \cite{Scala_Haslum_Thiebaux_Ramirez_2016_AIBR}, formally defined in the following. 
 A \textsl{relaxed state} $\hat{S}$ is
a function mapping each Boolean variable $v \in V_B$ to a subset of $\set{\bot, \top}$ and each numeric variable $x \in V_N$ to an interval 
$[\lb{x}, \ub{x}]$ with $\lb{x} \in \rational \cup \set{-\infty}$, $\ub{x} \in \rational \cup \set{+\infty}$, $\lb{x} \le \ub{x}$.  Intuitively, 
the set associated to each variable represents a superset of the values it can assume in the relaxed state. 
A relaxed state $\hat{S}$ can be extended to linear numeric expressions $\psi$ by defining  $\hat{S}(\psi)$ according to Moore's Interval Analysis \cite{moore2009introduction}:  
for  $x,y \in V_N$, $c \in \rational$,  $\hat{S}(x) = [\lb{x}, \ub{x}]$ and  $\hat{S}(y) = [\lb{y}, \ub{y}]$, we define:
\begin{gather*}
    \hat{S}(c) = [c, c], \qquad \hat{S}(x + c) = [\lb{x} + c, \ub{x} + c],\\
    \hat{S}(x + y) = [\lb{x} + \lb{y}, \ub{x} + \ub{y}], \qquad \hat{S}(x - y) = [\lb{x} - \ub{y}, \ub{x} - \lb{y}],\\
    \hat{S}(c \cdot x) = [\min(c \cdot \lb{x}, c \cdot \ub{x}), \max(c \cdot \lb{x}, c \cdot \ub{x})].
\end{gather*}
The \textsl{convex union} of two intervals $[\lb{x}, \ub{x}]$ and $[\lb{y}, \ub{y}]$ is denoted and computed as $[\lb{x}, \ub{x}] \sqcup [\lb{y}, \ub{y}] = [\min(\lb{x}, \lb{y}), \max(\ub{x}, \ub{y})]$.
The {\sl convex union of two relaxed states} $\hat{S}_1$ and $\hat{S}_2$ is the relaxed state $\hat{S} = \hat{S}_1 \sqcup \hat{S}_2$  such that
\begin{enumerate}
    \item for each $v \in V_B$, 
$\hat{S}(v) = \hat{S}_1(v) \cup \hat{S}_2(v)$, and 
\item for each $x \in V_n$, $\hat{S}(x) = \hat{S}_1(x) \sqcup \hat{S}_2(x)$. 
\end{enumerate}

Consider a relaxed state $\hat{S}$.
We say that $\hat{S}$ {\sl satisfies}
\begin{enumerate}
    \item {\sl a propositional condition} $v = \top$ (resp. $v = \bot$) if $\top \in \hat{S}(v)$ (resp. $\bot \in \hat{S}(v)$),
    \item {\sl a numeric condition} $\psi \unrhd 0$ if $\hat{S}(\psi) = [\lb{\psi}, \ub{\psi}]$ and $\ub{\psi} \unrhd 0$,
\end{enumerate}    
and is naturally extended to sets of conditions. An  action $a$ is {\sl  executable} in the relaxed state $\hat{S}$ if $\hat{S}$ satisfies the preconditions of $a$.
If $a$ is executable in $\hat{S}$, the {\sl result of executing $a$ in $\hat{S}$} is the relaxed state $\hat{S}' = \res(a, \hat{S})$ such that:
\begin{enumerate}
    \item For each $v \in V_B$, $(i)$ if $v := \top \in \eff(a)$ (resp. $v := \bot \in \eff(a)$), then $\hat{S}'(v) = \hat{S}(v) \cup \set{\top}$ (resp. $\hat{S}'(v) = \hat{S}(v) \cup \set{\bot}$), and $\hat{S}'(v) = \hat{S}(v)$ otherwise. 
    \item For each numeric variable $x \in V_N$,
\begin{enumerate}
    \item if $x := \psi \in \eff(a)$ is a general assignment, then $\hat{S}'(x) = \hat{S}(x) \sqcup \hat{S}(\psi)$,
    \item if $x \pluseq \psi \in \eff(a)$ is a linear increment, then $\hat{S}'(x) = [\lb{x}',\ub{x}']$, where, given $\hat{S}(x) = [\lb{x},\ub{x}]$, $\hat{S}(\psi) = [\lb{\psi},\ub{\psi}]$ and $\hat{S}(x + \psi) = [\lb{x+\psi},\ub{x+\psi}]$,
    $$
    \lb{x}' = \begin{cases}
        - \infty & \text{if } \lb{\psi} < 0 \text{ and $a$ is eligible for rolling},\\
        \lb{x + \psi} & \text{otherwise},
    \end{cases}
    $$
    $$
    \ub{x}' = \begin{cases}
        +\infty & \text{if } \ub{\psi} > 0 \text{ and $a$ is eligible for rolling},\\
        \ub{x+\psi}& \text{otherwise},
    \end{cases}
    $$
    \item otherwise, $\hat{S}'(x) = \hat{S}(x)$. 
\end{enumerate}
\end{enumerate}
Finally, let $A' \subseteq A$ be a set of actions, each executable in $\hat{S}$. The {\sl parallel execution} of the actions in $A'$ results in the relaxed state $\hat{S}$ if $A' = \emptyset$, or otherwise
$$\hat{S}' = \res(A', \hat{S}) = \bigsqcup_{a \in A'} \res(a,\hat{S}).$$

\newcommand{\supporters}[1]{\Call{Supp}{#1}}
\newcommand{\action}[1]{\Call{ActOfSupp}{#1}}
\begin{algorithm}[t]
\caption{\textsc{computePatternARPG}. \\
Input: a starting state $S$, a set $A$ of actions, a set $G$ of goal conditions. \\
Output: An elementary pattern or \textsc{failure} in case $G$ is not reachable from $S$ using the actions in $A$. \label{alg:arpg-complete}
}
\begin{algorithmic}[1]
\Function{$\textsc{computePatternARPG}$}{$S, A, G$}
    \State $\arpg \gets \Call{computeARPG}{S, A, G}$
    \State \Return \textsc{computePattern($\arpg$)} \label{arpg2pattern}
\EndFunction

\Function{$\textsc{computeARPG}$}{$S, A, G$}
    \State $\arpg \gets \epsilon$
    \State $\aleft \gets \supporters{A}$
    \State $\hat{S} \gets \Call{Relax}{S}$
    \While{(\textsc{True})}
        \State $\alayer \gets \set{a \mid a \in \aleft, \text{$\hat{S}$ satisfies $\pre(a)$}}$ 
        \If{($\alayer = \emptyset$)}
            \If{($\hat{S}$ does not satisfy $G$)}
            \State \Return \textsc{failure}
            \Else 
            \State \Return $\arpg$ \label{arpgreturn}
            \EndIf
        \EndIf
        \State $\hat{S} \gets res(\alayer, \hat{S})$ \label{alg:arpg-complete:res}
        \State $\alayer' \gets \set{\action{a} \mid a \in \alayer} \setminus \set{\action{a} \mid a \in \supporters{A} \setminus \aleft}$
        \State $\aleft \gets \aleft \setminus \alayer$
        \State $\arpg \gets \arpg;\alayer'$
    \EndWhile
\EndFunction
\end{algorithmic}
\end{algorithm}

To compute the resulting relaxed state when different actions have \textsl{interfering numeric effects} (i.e., $x := \psi$ and $y := \psi'$ with $y$ in $\psi$ and $x$ in $\psi'$, see \citet{Aldinger_Mattmüller_Göbelbecker_2015}), \citet{Scala_Haslum_Thiebaux_Ramirez_2016_AIBR} introduced a compilation of an action to its \textsl{supporters}. Let $a$ be an action, the set of supporters of $a$ is the set of actions
\begin{flalign*}
    \supporters{a} &= \set{\tuple{\pre(a) \cup \set{\psi > 0}, \set{x:= \psi}}, \tuple{\pre(a) \cup \set{-\psi > 0}, \set{x:= \psi}} \mid x := \psi \in \eff(a)}\\ 
    & \cup \set{\tuple{\pre(a), \set{x := \top \in \eff(a)} \cup \set{x := \bot \in \eff(a)}}}.
\end{flalign*}
If $a' \in \supporters{a}$, we denote by $\action{a'}$ the action $a$. If $A$ is a set of actions, then $\supporters{A} = \bigcup_{a \in A} \supporters{a}$

An \textsl{Asymptotic Relaxed Planning Graph} (\arpg) is a sequence $A_1;A_2;\dots;A_k$ of pairwise disjoint sets of actions, called {\sl layers}.  Let $\hat{S}_0$ be a relaxed state. Executing in parallel each set of actions in the sequence $A_1;A_2;\dots;A_k$, one after the other in $\hat{S}_0$ {\sl induces} the sequence of $k+1$ relaxed states $\hat{S}_0;\hat{S}_1;\dots;\hat{S}_k$ such that, for $i \in \set{1,2,\ldots,k}$,$\hat{S}_{i}$ $(i)$ is undefined if either $A_i$ is not executable in $\hat{S}_{i-1}$ or $\hat{S}_{i-1}$ is undefined, and $(ii)$ is the result of executing $A_i$ in $\hat{S}_{i-1}$ otherwise. 
Algorithm~\ref{alg:arpg-complete} shows the construction of the \arpg starting from a state $S$ using the actions in $A$ having $G$ as set of goal conditions.
The algorithm starts from the relaxed version $\hat{S} = \Call{Relax}{S}$ of the state $S$ which maps
\begin{enumerate}
    \item each Boolean variable $v$ to $\set{S(v)}$, and
\item each numeric variable $x$ to $[S(x),S(x)]$.
\end{enumerate} 
Then, the algorithm constructs the first layer of the \arpg by adding to $\alayer$ all the supporters of the actions in $A$ executable in $\hat{S}$. When all executable supporters are added to $\alayer$, the relaxed state is updated, and the actions corresponding to the supporters in $\alayer$ not selected at previous layers are added to the \arpg, until no further layers can be constructed. If the last relaxed state satisfies all the conditions in $G$, the pattern computed by \textsc{computePattern($\arpg$)} is returned (Line~\ref{arpg2pattern}), otherwise, a \textsc{failure} value signals that it is not possible to reach a goal state from $S$ using the actions in $A$.
The pattern returned by \textsc{computePattern($\arpg$)}
is obtained by extending the partial order defined by the \arpg to a total order. In more details, given an \arpg $A_1;A_2;\dots;A_k$ the corresponding pattern $\pattern$ is obtained  by ensuring that for any pair of actions $a$ and $a'$ with $a$ occurring before $a'$ in $\pattern$, $a \in A_i$ and $a' \in A_j$, we have either  $A_i = A_j$ or $A_i$ occurring before $A_j$ in the \arpg (i.e., $i \le j$). To uniquely characterize the returned pattern, we also assume that each pair of actions in the same layer $A_i$ of the \arpg are lexicographically ordered by \textsc{computePattern(\arpg)}.
On the other hand, if \textsc{computeARPG}($S,A,G$) returns \textsc{failure}, then $\Pi$ admits no valid plan \cite{Scala_Haslum_Thiebaux_Ramirez_2016_AIBR}.

Notice that each pattern returned by \textsc{computePatternARPG}$(S,A,G)$ is elementary but not necessarily complete.
Indeed, when the condition $\alayer = \emptyset$ is met and the algorithm exits, some actions may still remain in $\aleft$. These actions are not executable in any state reachable from $S$ via a plan executable in $S$. In particular, when $S = I$, such actions will never become executable and can be safely removed from the whole set $A$ of actions of $\Pi$, making the returned pattern complete. For this reason, if $\pattern = \textsc{computePatternARPG}(I,A,G)$ and the planning problem $\Pi$ admits a valid plan, then there exists a bound $n$ such that $\Pi^\pattern_n$ is satisfiable \cite{DBLP:conf/aaai/CardelliniGM24}. 

\section{Exploiting Search in Symbolic Numeric Planning with Patterns}\label{sec:pattyd}

Consider a numeric planning problem $\Pi = \langle X, A, I, G \rangle$. In this section, we begin by presenting a motivating example, while also introducing the key intuitions and ideas behind our proposed solutions (Subsection~\ref{subsec:pattyd-intuitions}). We then present a general procedure $\pattyd(\Pi)$ that incorporates these ideas, prove its correctness, and identify conditions under which it is also complete (Subsection~\ref{subsec:pattyd-procedure}). 

\subsection{Motivating Example and Intuitions}\label{subsec:pattyd-intuitions}
\newcommand{\At}{\mathit{at}}
\newcommand{\OnSale}{\mathit{onSale}}
\newcommand{\bought}{\mathit{bought}}
\newcommand{\cash}{\mathit{cash}}
\newcommand{\price}{\mathit{price}}
\newcommand{\travel}{\mathit{travel}}
\newcommand{\buy}{\mathit{buy}}
\newcommand{\sell}{\mathit{sell}}

In the market trader domain from the last \ipc~\cite{ipc2023}, the aim of the planning agent is to make money by buying, transporting and selling goods in different markets. We consider a simplified version of it, in which  we assume $(i)$ to have a single object that can be either bought or sold in a set $M$ of markets, $(ii)$ that moving between markets involves traveling through a set $L$ of locations (with $L$ a superset of $M$), and $(iii)$ that the aim of the agent is to have either more money or more objects or both. We can model this scenario in \textsc{pddl2.1} by introducing the state variables
\begin{enumerate}
\item $\At(loc) \in V_B$, representing the fact that the current location of the agent is $loc$, for each $loc \in L$,
\item $\bought \in V_N$, denoting the number of objects possessed,
\item $\cash \in V_N$, denoting the amount of money possessed,
\item $\price(mkt) \in V_N$, denoting the price at which an object can be bought or sold in the market~$mkt$, for each $mkt \in M$.
\end{enumerate}
Assuming locations are numbered from 1 to $m$, i.e., $L = \set{1, 2, \dots, m}$, that the only two markets are the first and the last location, i.e., $M = \set{1,m}$, and that it is possible to travel only from location $i$ to location $i+1$ and vice-versa, the actions of travelling from one location to the other, buying and selling an object in a market can be formalized as follows:
$$
\setlength{\arraycolsep}{0pt}
\begin{array}{rl}
    \travel(f,t) : \tuple{&\set{\At(f) = \top}, \set{\At(t) \asseq \top, \At(f)\asseq \bot}}, \\
    \buy(i) : \tuple{&\set{\At(i), 
    \cash \ge \price(i)},  \set{\bought \pluseq 1, \cash \minuseq \price(i)}}, \\ 
    \sell(i) : \tuple{&\set{\At(i), \bought \ge 1}, \set{\bought \minuseq 1, \cash \pluseq \price(i)}}.
\end{array}
$$
in which $i$ takes values in the set $\set{1,m}$, $f$ and $t$ in the set $\set{1,2,\ldots,m}$ with $|f - t| = 1$. If we further assume that in the initial state $I$ the following atomic formulas are satisfied:
\begin{equation}\label{eq:ex-init}
    \At(1),\quad \cash = 100,\quad \bought = 0,\quad \price(1) = 10,\quad \price(m) = 100;
\end{equation}
then the \arpg{} construction from the initial state, leads to the following pattern $\pattern$:
\begin{equation}
    \label{eq:ex-pattern}
\begin{array}{rl}
\pattern = &
\buy(1); \sell(1); \travel(1,2);\travel(2,1);\travel(2,3);\ldots; \\ 
& \travel(m-1,m-2); 
\travel(m-1,m); \buy(m);\sell(m);\travel(m,m-1).
\end{array}
\end{equation}
Then, in the $\pattern$-encoding of $\Pi$,
$\mathrm{pre}^\pattern(A)$ contains, for the actions $\travel(f,t)$, $\buy(m)$ and $\sell(m)$, formulas entailing ($f \in\set{2,3,\ldots,m}, t \in\set{1,2,\ldots,n-1}$)
$$
\begin{array}{c}
    \buy(1) > 0 \implies \At(1), \\ 
    \buy(1) > 0 \implies \cash 
    \ge \price(1) \times \buy(1), \\

    \sell(1) > 0 \implies \At(1), \\
    \sell(1) > 0 \implies \bought +  \buy(1) \ge 1, \\
    \sell(1) > 1 \implies \bought +  \buy(1) - (sell(1) - 1) \ge 1, \\

    \travel(1,2) > 0 \implies \At(1), \\
    \travel(f,f+1) > 0 \implies \At(f) \vee \travel(f-1,f) > 0, \\
    \travel(t+1,t) > 0 \implies \At(t+1) \vee \travel(t,t+1), \\

    \buy(m) > 0 \implies \At(m) \vee \travel(m-1,m) > 0, \\
    \buy(m) > 0 \implies \cash + \price(1) \times (\sell(1) - \buy(1)) \ge \price(m) \times \buy(m), \\
    
    \sell(m) > 0 \implies \At(m) \vee \travel(m-1,m) > 0, \\
    \sell(m) > 0 \implies \bought + \buy(1) - \sell(1) + \buy(m) \ge 1\\
    \sell(m) > 1 \implies \bought + \buy(1) - \sell(1) + \buy(m) - (sell(m) - 1) \ge 1.
\end{array}
$$
For each  $f,t \in\set{1,2,\ldots,m}$, $|f-t| = 1$, the action $\travel(f,t)$ is not eligible for rolling, and thus 
$$ \begin{array}{c}
     \travel(f,t) = 0 \vee
     \travel(f,t) = 1,
\end{array} $$ 
belongs to $\mathrm{amo}^\pattern(A)$.
Finally, the frame axioms in $\mathrm{frame}^\pattern(X)$, for each $i \in \set{2,3,\ldots,m-1}$,  entail:
$$ \begin{array}{c}
    \cash' = \cash + \price(1) \times (\sell(1) - \buy(1)) + \price(m) \times (\sell(m) - \buy(m)), \\
    \bought' = \bought + (\buy(1) - \sell(1)) + (\buy(m) - \sell(m)), \\
    \At'(1) \equiv (\At(1) \wedge  \travel(1,2) = 0) \vee \travel(2,1) > 0, \\
    \At'(i) \equiv ((\At(i) \vee \travel(i-1,i) > 0) \wedge \travel(i,i+1) = 0) \vee \travel(i+1,i) > 0, \\
    \At'(m) \equiv (\At(m) \vee \travel(m-1,m) > 0) \wedge \travel(m,m-1) = 0.
\end{array} $$ 
As the frame axioms make clear, the $\pattern$-encoding allows in a single state transition $(i)$ the multiple consecutive execution of the same action, as in the rolled-up encoding \cite{Scala_Ramirez_Haslum_Thiebaux_2016_Rolling}, and $(ii)$ the combination of multiple even contradictory effects on a same variable by different actions, as in the \re{} encoding \cite{DBLP:conf/ijcai/BofillEV17}. 

If the goal  is to be in a state with at least five objects and with 500 tokens i.e., 
\begin{equation}\label{eq:G1}
G = \set{\bought \ge 5, \cash = 500}
\end{equation}
then a possible plan is to buy 10 objects in the first market, travel to the last market and there sell 5 objects, i.e.,
\begin{equation}
    \label{eq:G1-plan}
\begin{array}{c}
\pi = \buy(1)^{10}; \travel(1,2);\travel(2,3);\ldots; \travel(m-1,m); \sell(m)^5, 
\end{array}
\end{equation}
where, for each action variable $a$ and natural number $i \ge 0$, $a^i$ denotes the sequence $a;a;\ldots;a$ of length $i$.
Luckily, such a plan corresponds to a model of the $\pattern$-encoding of $\Pi$ with $n=1$.
    However, if the goal is to have at least $50$ objects and 500 tokens, i.e., if 
    \begin{equation}\label{eq:G2}
G = \set{\bought \ge 50, \cash = 500},
    \end{equation}
 the $\pattern$-encoding of $\Pi$ requires a bound $n=m+1$. 
    This is because in this case 
\begin{enumerate}
    \item the shortest valid plan is to buy 10 objects in the first market and travel to the last market (as in the previous case), there sell all the 10 objects for 1000 tokens and then travel back to the first market in order to buy 50 objects, and
    \item returning to the initial location  necessitates a plan where $\travel(t+1,t)$ is executed before $\travel(t,t-1)$ for $t \in \set{2,3,\ldots,m-1}$, while their order in the pattern $\pattern$ of (\ref{eq:ex-pattern}) is the opposite.
\end{enumerate}
Thus, in the \spp approach proposed by \citet{DBLP:conf/aaai/CardelliniGM24}, the formula $\Pi^\pattern_n$ allowing to find a valid plan
\begin{enumerate}
    \item if the goal is (\ref{eq:G1}), has $n=1$, $2\times m + 2$ action variables ($2\times m - 2$ to move between the markets and $4$ to buy and sell in the first and last market), and $2 \times (m+4)$ state variables ($m$ for the location of the agent, $bought$, $cash$, and $2$ for the prices in the first and last market, copied in $\mX$ and $\mX'$), and
    \item if the goal is (\ref{eq:G2}), has $n=m+1$, $(m+1)\times (2\times m + 2)$ action variables and $(m+2)\times (m+4)$ state variables.
\end{enumerate}


The problem highlighted by the motivating example arises because $(i)$ an elementary and complete  pattern $\pattern$ is computed only once at the beginning, starting from the initial state, and $(ii)$ then exploited at every step $i \in \set{0,1,\ldots,n}$ in the $\pattern$-encoding of $\Pi$ with bound $n$ (the formula $\Pi^\pattern_n$). 
Indeed, such procedure does not take into account the fact that if 
at some step $i < n$ in the search we reach an intermediate state $P$ closer then $I$ to a goal state, it may be better to
\begin{enumerate}
    \item  recompute the pattern in $P$ and use the new one for the subsequent search, 
    \item  start the new search from $P$, and
    \item if starting from $I$, refine the pattern that allowed to reach $P$ by removing the actions which are not necessary to reach $P$ from $I$.
\end{enumerate}
Of course, independently from whether the starting state $S$ is set to $P$ or to $I$, we would like to use a pattern that allows to {\sl reach} a goal state from $S$, i.e., a pattern that
covers a plan leading to a goal state when executed from $S$.
A pattern $\pattern$ \textsl{covers a plan} $\pi$
if $\pattern$ is a supersequence of the pattern of $\pi$.
A sequence of actions 
$\pattern$ is the \textit{pattern of a plan} $\pi$
if 
$\pattern$ is obtained from 
$\pi$ by replacing consecutive occurrences of each action 
$a$ eligible for rolling with a single instance of $a$.
If a pattern $\pattern$ covers a valid plan, the formula  $\Pi^\pattern$ (i.e., $\Pi^\pattern_n$ with $n=1$) is satisfiable.

\begin{theorem}\label{th:plan-pattern}
    Let $\Pi$ be a numeric planning problem. 
    Let $\pattern$ be a pattern covering a valid plan.  Let $\Pi^\pattern$ be the formula (\ref{eq:enc-sat}) with $n=1$. $\Pi^\pattern$ is satisfiable. 
\end{theorem}

\begin{proof}
Assume $\pi$ is a valid plan covered by $\pattern$.
First, considering $\pi$ as a pattern, the pattern $\pi$-encoding  of $\Pi$ with bound $n=1$ (i.e., $\Pi^\pi_1$) is satisfiable. Then, if $\pattern_\pi$ is the pattern of $\pi$,  $\Pi^{\pattern_\pi}_1$ is satisfiable. Finally, $\pattern$ is a supersequence of  $\pattern_\pi$, and thus $\Pi^{\pattern}=\Pi^{\pattern}_1$ is satisfiable.
\end{proof}

Of course, the challenge of finding a pattern that allows to reach a goal state, or, equivalently, that covers a valid plan is as complex as finding the plan itself. In practice, starting with an arbitrary pattern $\pattern$, we can iteratively extend and/or simplify it at each step, also checking whether the pattern allows reaching a state heuristically believed to be closer to a goal state. The pattern $\pattern$ can be arbitrarily extended and/or simplified, but, for effectiveness, this needs some care, since
\begin{enumerate}
    \item each newly introduced action in the pattern adds another variable to the encoding, thereby increasing the search space of the solver, and
    \item each removed action reduces the set of plans being covered, and possibly also the set of states that can be reached with it.
\end{enumerate}
Indeed, different patterns —even when one is merely a permutation of the other— usually cover different plans, and various strategies can be devised for selecting which patterns to try.
Starting from a pattern computed in the initial state, the general procedure --presented in the next Subsection-- iteratively extends and simplifies the previously used pattern, possibly at each iteration, till all the goals are satisfied. Such  procedure, compared to the \textsc{sota} procedure based on the  $\Pi^\pattern_n$ encoding in (\ref{eq:enc-sat}), incorporates the following three ideas:
\begin{enumerate}
    \item With patterns, in (\ref{eq:enc-sat}) we do not need the intermediate state variables $\mX_1,\ldots, \mX_{n-1}$. Given an initially computed pattern $\pattern_h$, we can check whether a goal state is reachable with $n$-steps by just concatenating $\pattern_h$     
    $n$ times, and have only two copies $\mX$ (for the initial state) and $\mX'$ (for the goal state) of the set of state variables.
    \item In the procedure outlined in the previous item, it is not necessary to keep the same $\pattern_h$ at each iteration.
    Given a pattern $\pattern_g$ allowing to reach a state $P$ which is believed to be closer to the goal --e.g., because   $P$ satisfies some new goal in $G$-- we can $(i)$ dynamically recompute the pattern $\pattern_h$, and then $(ii)$ iterate the procedure with pattern $\pattern = \pattern_g;\pattern_h$, i.e., the concatenation of $\pattern_g$ and $\pattern_h$, 
    exiting when all the goals in $G$ are satisfied.
    \item In the procedure outlined in the previous item,  given a pattern $\pattern_g$ allowing  to compute a plan $\pi$ leading to a state $P$ in which $\pattern_h$ is recomputed, at each iteration we can 
    \begin{enumerate}
        \item either start searching from $P$ with $\pattern_h$ and $\pattern_g = \epsilon$, 
        \item or search again from the initial state, possibly refining $\pattern_g$ by eliminating the actions which are not necessary to reach $P$. 
    \end{enumerate}
\end{enumerate}

In our example, considering the goal (\ref{eq:G2}),
starting from the initial state and the empty pattern $\pattern_g = \epsilon$ and with pattern $\pattern = \pattern_g; \pattern_h = \pattern_h$ as in (\ref{eq:ex-pattern}), we can 
\begin{enumerate}
    \item determine that we can reach a state $P$ satisfying the goal to have 500 tokens but with only 5 objects,  by travelling through all the locations up to the market $m$, and there selling 5 of the 10 objects bought in the first market, with the plan $\pi$ in (\ref{eq:G1-plan}), and
    \item assuming the computed plan is $\pi$ and that $P=res(\pi,I)$,
    recompute the pattern $\pattern_h$ with $P$ as initial state, which, with the \arpg construction, will be
    \begin{equation}
    \label{eq:ex-pattern2}
\begin{array}{rl}
\pattern_h = &
\buy(m); \sell(m); \travel(m,m-1); \travel(m-1,m-2);\\
& \travel(m-1,m); \ldots; 
\travel(2,1); 
\buy(1);\sell(1);\travel(1;2).%
\end{array}
\end{equation}
 and, finally,
\item 
using  pattern (\ref{eq:ex-pattern2}), determine that also the last goal $\bought \ge 50$ can be satisfied by selling the 5 objects still owned, traveling back to the first market and there buying 50 objects. This plan can be determined checking the satisfiability of
\begin{enumerate}
    \item either  the encoding of $\Pi$ with the pattern $\pattern_h$ and $P$ as starting state,
    \item or the encoding of $\Pi$ with the pattern $\pattern_g;\pattern_h$ and $I$ as starting state, where $\pattern_g$ is the pattern used to reach $P$ from $I$ computed in the previous iteration (i.e., the previously computed $\pattern_h$), 
    \item or --by exploiting the plan $\pi$ as pattern to reach $P$-- the encoding of $\Pi$ with the pattern $\pi;\pattern_h$ and $I$ as starting state.
\end{enumerate}
\end{enumerate}
Notice that up to this point we assumed that the pattern $\pattern_h$ used at each iteration is complete, since computed with \textsc{computePatternARPG}$(S, A, G)$. Indeed, completeness of $\pattern_h$ is not necessary, and we may try to see whether we can make some progress beyond $P$ with an incomplete pattern corresponding to an \arpg having only some of the actions in $A$, but whose last induced relaxed  state still satisfies the goals in $G$. In our example, if the goal is to have at least 5 objects and 50 tokens, i.e., 
\begin{equation}\label{eq:G3}
G = \set{\bought \ge 5, \cash = 50}
\end{equation}
starting from the initial state, we may consider the \arpg
\begin{equation}
\label{eq:inc-arpg}    
\begin{array}{c}
\set{\buy(1)}
\end{array}
\end{equation}
whose corresponding pattern
\begin{equation}
    \label{eq:inc-pattern-G1-plan}
\begin{array}{c}
\pattern_h = \buy(1)
\end{array}
\end{equation}
allows to find the valid plan consisting in just buying 5 objects, without resorting to any complete pattern with $2\times m + 2$ actions.


\subsection{Search in Symbolic Numeric Planning with Patterns}\label{subsec:pattyd-procedure}

As anticipated, we are going to present a procedure that, starting from a state $S$ (initially set to the initial state $I$) and with  pattern $\pattern$, tries to reach a new state $P$ closer to the goal, iterating until $P$ is a goal state. As a consequence,  we do not need variables representing intermediate reachable states, and we will use just the following sets of variables:
\begin{enumerate}
    \item $\mX$ for the set of state variables representing the state $S$ in which the search starts, 
    \item $\mA^\pattern$ for the set of action variables, containing one distinct variable per action occurrence in~$\pattern$,
    \item $\mX'$ for the set of state variables representing the condition on the resulting state.
\end{enumerate}

\begin{algorithm}[t]
\caption{\pattyd algorithm. \\
Input: a numeric planning problem~$\Pi= \langle X, A, I,G\rangle$.\\
Output: a valid plan for $\Pi$. 
}\label{alg:pattyd}
\begin{algorithmic}[1]
\Function{\pattyd}{$\Pi$} 
    \State $P \leftarrow S \leftarrow I$ \label{alg:d:S-init}
    \State $\pattern_P \leftarrow \pattern_g \leftarrow\epsilon$ \label{alg:d:pattern-g-init}
    \State $\mu_P \leftarrow \emptyset$ \label{alg:d:mu-G-init}
    \State $\pattern_h,{GF}_P \leftarrow \ptg(P,A,G)$ \label{alg:d:pattern-h-init}
    \State $n \leftarrow 0$ \label{alg:d:n-init}
    \While{(\textsc{True})}
        \State $n \leftarrow n+1$ \label{alg:d:n-inc}
        \State $\pattern \leftarrow \pattern_g; \pattern_h$ \label{alg:d:concat} 
        \State $\Pi^{\pattern}_{S,P} \gets \mS(\mX) \wedge \mT^{\pattern}(\mX,\mA^{\pattern},\mX') \wedge {\mGF}_P(\mX') \le \max(P({\mGF}_P(\mX)) - \varepsilon,0)$ \label{alg:d:pi}         
        \State $\mu \leftarrow \textsc{Solve}(\Pi^{\pattern}_{S,P})$ \label{alg:d:solve}
        \If{($\mu({\mGF}_P(\mX'))=0$)} \label{alg:d:plan-cond}
            \State \Return 
            $\ite(S=I,
            \textsc{getPlan}(\mu, \pattern),
            \textsc{getPlan}(\mu_P\cup\mu, \pattern_P;\pattern))$
            \label{alg:d:plan}
        \ElsIf{($\mu({\mGF}_P(\mX')) < P({\mGF}_P(\mX))$)} \label{alg:d:S-cond}
            \State $P \leftarrow res(\textsc{getPlan}(\mu, \pattern),S)$ \label{alg:d:P-get}
            \State $\pattern_P \leftarrow \ite(S=I,\pattern,\pattern_P;\pattern)$ \label{alg:d:patternP}
            \State $\mu_P \leftarrow \ite(S=I,\mu,\mu_P \cup \mu)$ \label{alg:d:muP}
            \State $S,\pattern_g \leftarrow \stp(I,P,\mu_P,\pattern_P)$ \label{alg:d:g-pattern}
            \State $\pattern_h, {GF}_P \leftarrow \ptg(P,A,G)$ \label{alg:d:h-pattern} 
        \Else
            \State $S,\pattern_g \leftarrow \stpn(I,P,\mu_P,\pattern_P,n)$ \label{alg:d:g-n-pattern}
            \State $\pattern_h \leftarrow \ptgn(S,P,A,G,n)$ \label{alg:d:h-n-pattern} 
        \EndIf
    \EndWhile
\EndFunction
\end{algorithmic}
\end{algorithm}

Algorithm~\ref{alg:pattyd} presents the procedure, dubbed \pattyd, exploiting the ideas informally presented in the previous subsection. 
In \pattyd, 
\begin{enumerate}
    \item $P$ denotes the computed {\sl intermediate state}, reachable from the initial state and, so far, the closest to satisfying all goals. At the beginning, $P$  is set to the initial state $I$.
    \item $S$ is the  state in which the search is started, at the beginning equal to $I$.
    \item $\pattern_P$ and $\pattern_g$ store a pattern 
    allowing to reach $P$ from $I$ and $S$ respectively. Both $\pattern_P$ and $\pattern_g$
    are set to the empty pattern at the beginning.
    \item $\mu_P$ is the assignment to the variables in $\pattern_P$ corresponding to a plan allowing to reach $P$ from $I$, initially set to the empty assignment.
    \item $\ptg(P,A, G)$ returns 
    \begin{enumerate}
        \item a non-empty pattern $\pattern_h$ computed from $P$ which (hopefully) should allow determining a new intermediate state $P'$ closer than $P$ to a goal state, and 
        \item a {\sl goal value function} $GF_P$ mapping each state to $\real^{\geq 0}$, meant to represent how close is the state to a goal state, and formally returning, for each state $P'$, 
    \begin{enumerate}
        \item the value $0$ if $P'$ is a goal state, and
        \item a positive value otherwise. 
    \end{enumerate} 
    We assume that each action occurrence in $\pattern_h$ is a fresh copy of an action in $A$, this to avoid conflicts with other copies of the action, possibly generated in some previous iteration.
    As we have already seen, the pattern construction can immediately reveal if the goal $G$ is not reachable from state $P$ using the actions in $A$ (which, in the initial call with $P=I$ amounts to determine that the problem $\Pi$ is not solvable). To simplify the algorithm, we omit the check for this case.
    \item $P({\mGF}_P(\mX))$ returns the value of the goal value function in state $P$. We write ${\mGF}_P(\mX)$ instead of $GF_P$ --and, more in general, we use the calligraphic style-- when explicitly indicating the set of variables --in this case $\mX$-- on which the expression depends.
    \item $n$ is the number of iterations, initially set to $0$.
    \item For any state $S$, $\mS(\mX)$ is the formula imposing that $S$ is the starting state, i.e.,
    $$ 
    \mS(\mX) = \bigwedge_{v: S(v) = \top} v  \wedge \bigwedge_{w: S(w) = \bot} \neg w  \wedge \bigwedge_{x, k: S(x) = k} x = k.
    $$
    \item ${\mGF}_P(\mX')$ is obtained from ${\mGF}_P(\mX)$ by replacing each state variable $x \in \mX$ with its corresponding next-state variable $x' \in \mX'$. If $\mu$ is an assignment to the variables in $\mX'$, $\mu({\mGF}_P(\mX'))$ returns the $\mu$ evaluation of ${\mGF}_P(\mX')$, i.e., the goal value of the state $P'$ such that, for each variable $x \in \mX$, $P'(x) = \mu(x')$.
    \item For any state $S$, pattern $\pattern$, and state $P$, the formula $\Pi^{\pattern}_{S,P}$ is defined as:
    \begin{equation}\label{eq:pi-s-pattern}
    \Pi^{\pattern}_{S,P} = \mS(\mX) \wedge \mT^{\pattern}(\mX, \mA^{\pattern}, \mX') \wedge {\mGF}_P(\mX')  \le \max(P({\mGF}_P(\mX)) - \varepsilon, 0) 
    \end{equation}
    This formula is satisfiable when there exists a state $P'$ that is reachable with the pattern $\pattern$, and whose associated goal value is either 0 or smaller than that of $P$ by at least a fixed constant $\varepsilon > 0$ that we assume  $< 1$. The constant $\varepsilon$ is introduced to preclude asymptotic behaviours with infinite chains of non-goal states, each having a strictly decreasing associated goal value.
    \item    
    $\textsc{Solve}(\Pi^{\pattern}_{S,P})$ calls an \smt solver which returns a model of the given formula if it is satisfiable, and an assignment 
    $\mu$ with $\mu({\mGF}_P(\mX')) \ge P({\mGF}_P(\mX))$, otherwise. Any assignment satisfying $\mS(\mX)$ satisfies such condition.
    \item $\textsc{getPlan}(\mu, \pattern)$ returns the sequence of actions $\pi$ ordered as in $\pattern$, each action $a$ repeated $\mu(a)$ times, and analogously $\textsc{getPlan}(\mu_P\cup\mu, \pattern_P;\pattern)$.
    \item  $\stp(I,P,\mu_P, \pattern_P)$ returns a pair representing the state $S$ and pattern $\pattern_g$ to be used at the next iteration. Admissible pairs are:
    \begin{enumerate}
    \item $S = P$ and $\pattern_g = \epsilon$: at the next iteration, the search starts from the state $P$ and $\pattern = \pattern_g;\pattern_h=\pattern_h$, or
    \item $S = I$ and $\pattern_g$ is a pattern  covering a plan allowing to reach the state $P$, typically a subsequence of $\pattern_P$, e.g., $\pattern_P$ itself.
    \end{enumerate}
    \item 
    $\stpn(I,P,\mu_P,\pattern_P,n)$ returns a pair representing the state $S$ and pattern $\pattern_g$ to be used at the next iteration. As before, admissible pairs are:
    \begin{enumerate}
    \item $S = P$ and $\pattern_g = \epsilon$, as in the call to $\stp(I, P, \mu_P, \pattern_P)$: in this case, the next iteration of the search starts from state $P$ with $\pattern = \pattern_h$; or
    \item $S = I$ and $\pattern_g$ is any pattern that covers a plan allowing to reach $P$.
    \end{enumerate}
    \item $\ptgn(S, P, A, G, n)$ returns a non-empty pattern $\pattern_h$. As before, we assume that each action occurrence in the returned pattern $\pattern_h$ is a fresh copy of an action in $A$. The selection of $\pattern_h$ at each iteration $n$ may follow different strategies for exploring the search space, taking into account the chosen starting state $S$, the closest-to-goal computed intermediate state $P$, the available actions $A$, and the goal conditions $G$.
    \end{enumerate}
\end{enumerate}
The definitions of the procedures $\stp(I,P,\mu_P,\pattern_P)$/$\stpn(I,P,\mu_P,\pattern_P,n)$ and the procedures $\ptg(P,A,G)$/$\ptgn(S,P,A,G,n)$ allow for the selection of a different starting state $S$ and patterns $\pattern_g$, $\pattern_h$ at each iteration, thereby yielding a general procedure that can be tailored to the domain at hand. Regardless of their specific definitions, in \pattyd, after the initialization of all the variables in Lines~\ref{alg:d:S-init}-\ref{alg:d:n-init}, we start iterating till a valid plan is found. At each iteration, after the counter $n$ is incremented (Line~\ref{alg:d:n-inc}),
\begin{enumerate}
     \item we set the  pattern $\pattern$ to  $\pattern_g;\pattern_h$ (Line~\ref{alg:d:concat}), and then check whether a state $P'$ with a lower than $P({\mGF}_P(\mX))$ goal value can be reached (Line~\ref{alg:d:solve}).
    \item If such a state $P'$ is a goal state, (condition $(\mu({\mGF}_P(\mX')) = 0)$ at Line~\ref{alg:d:plan-cond}) a plan is returned (Line~\ref{alg:d:plan}) and \pattyd{} terminates.
    \item If such a state $P'$ can be reached but it is not a goal state (condition $(\mu({\mGF}_P(\mX')) < P({\mGF}_P(\mX)))$ at Line~\ref{alg:d:S-cond}),
    we $(i)$ set $P$ to such state (Line~\ref{alg:d:P-get}); $(ii)$ set $\pattern_P$ and $\mu_P$ to be the pattern and assignment to the variables in $\pattern_P$ allowing to compute a plan for reaching the newly determined intermediate state $P$ from $I$ (Lines~\ref{alg:d:patternP}-\ref{alg:d:muP}); $(iii)$ determine the starting state $S$ and $\pattern_g$ to be used at the next iteration (Line~\ref{alg:d:g-pattern}); $(iv)$ recompute the pattern $\pattern_h$ to be used at the next iteration (Line~\ref{alg:d:h-pattern}); and 
    restart the loop,
    \item otherwise, a starting state $S$, pattern $\pattern_g$ and pattern $\pattern_h$,  are computed (Lines~\ref{alg:d:g-n-pattern}-\ref{alg:d:h-n-pattern}) and the loop is iterated once more.
\end{enumerate}


\pattyd{} correctness relies on the correctness of the encoding (for any pattern $\pattern$, each model of the pattern $\pattern$-encoding $\Pi^{\pattern}$ of $\Pi$ corresponds to a valid plan) and the given definitions. Its completeness critically depends on which starting state $S$ and  pattern $\pattern$ are selected at each iteration in the definition of $\Pi^{\pattern}_{S,P}$ at Line~\ref{alg:d:pi}. Indeed, given a valid plan $\pi$ of length $k$, 
there exists an assignment $\mu$ satisfying $\Pi^{\pattern}_{S,P}$ such that $\mu({\mGF}_P(\mX')) = 0$ if
\begin{enumerate}
    \item the starting state $S$ is set to $I$, and
    \item $\pattern$ is a {\sl $k$-complete} pattern, i.e., $\pattern$ is a supersequence of a pattern obtained by concatenating $k$ complete patterns.
\end{enumerate}

\begin{theorem}\label{th:ncompl-pattern}
    Let $\Pi$ be a numeric planning problem having a valid plan of length $k$. Let $\pattern$ be a $k$-complete pattern. For any state $P$,
    $\Pi^{\pattern}_{I,P}$ is satisfiable.
\end{theorem}
\begin{proof}
    A plan $\pi$ of length $k$ is a supersequence of the pattern of $\pi$ and is a subsequence of any $k$-complete pattern. Thus, $\pattern$ covers $\pi$ and $\Pi^\pattern = \Pi^\pattern_1$ is satisfiable by Thm.~\ref{th:plan-pattern}. Then, any model of $\Pi^\pattern$ satisfies also ${\mGF}_P(\mX') = 0$, thus also $\Pi^{\pattern}_{I,P}$ and the thesis follows.
\end{proof}

Thus, if at some iteration $n$ of $\pattyd(\Pi)$, the search starts from $S = I$ and the pattern~$\pattern$ is $k$-complete, then the models of $\Pi^{\pattern}_{S,P}$ include those corresponding to the valid plans of length $\leq k$.  
However, $\textsc{Solve}(\Pi^{\pattern}_{S,P})$ may not return any of these plans, since it is only required to find an assignment satisfying $\Pi^{\pattern}_{S,P}$, thus possibly returning a solution with $0 < \mu({\mGF}_P(\mX')) \le P({\mGF}_P(\mX)) - \varepsilon$. However, this behaviour is not possible if  $\pattyd(\Pi)$ is {\sl goal directed}, meaning that for any state $S$, pattern $\pattern$ and state $P$,
$\textsc{Solve}(\Pi^{\pattern}_{S,P})$ returns an assignment $(i)$ always
satisfying $\Pi^{\pattern}_{S,P}$, and $(ii)$ $\mG(\mX')$ whenever possible. How to make \pattyd goal-directed is dependent on the implementation of the solver and will be discussed in Section \ref{sec:impl-exper}.

\begin{theorem}\label{th:corr-compl} 
Let $\Pi$ be a numeric planning problem.
The procedure $\pattyd(\Pi)$ is correct. Moreover, if for every $k \geq 0$, there exists an iteration $n$ of $\pattyd(\Pi)$ such that the formula $\Pi^\pattern_{S,P}$ at Line~\ref{alg:d:pi} of $\pattyd(\Pi)$ is constructed with $S = I$ and a $k$-complete pattern $\pattern$, and
\begin{enumerate}
    \item either $\Pi^\pattern_{S,P}$ has $S = I$ and a $k$-complete $\pattern$ also for every iteration $m > n$, 
    \item or $\pattyd(\Pi)$ is goal directed, 
\end{enumerate}
then $\pattyd(\Pi)$ is also complete.
\end{theorem}

\begin{proof}
        Correctness follows from the correctness of the encoding. For completeness, assume there exists a valid plan $\pi$ of length $k$. Let $n$ be the first iteration in which  $\Pi^\pattern_{S,P}$ is constructed with $S = I$ and a $k$-complete pattern. Then, $\pi$ is a subsequence of $\pattern$ and thus there exists an assignment satisfying $\Pi^\pattern_{I,P}$ and all the goal conditions. Assuming condition 1 is satisfied, one such assignment exists also for every $m > n$, and, if $\mu$ is a model of $\Pi^\pattern_{I,P}$, 
        after at most  $\lceil\mu(\mGF_P(\mX'))/\varepsilon\rceil$ iterations, one assignment satisfying all the goal conditions will be computed and the corresponding valid plan returned. Assuming condition 2 is satisfied, at iteration $n$ one model of $\Pi^\pattern_{I,P}$ satisfying all the goal conditions will be computed and the corresponding valid plan returned.
\end{proof}



\section{Cautious, Brave, Reckless and Greedy Procedures}\label{sec:cautious-brave-reckless-greedy}
Let $\Pi = \tuple{X,A,I,G}$ be a numeric planning problem. Consider the procedure $\pattyd(\Pi)$ in Algorithm~\ref{alg:pattyd}, and let $P$ be the last computed intermediate state. The procedure can be customized in many ways, by defining, possibly at each iteration,
\begin{enumerate}
    \item 
    a function --in $\pattyd(\Pi)$ abstractly modelled with the goal value function ${\mGF}_P(\mX)$-- allowing us to decide when a state $P'$ is ``closer" than  $P$ to a goal state, and ultimately, which assignment $\mu$ has to be returned by $\textsc{Solve}(\Pi^{ \pattern}_{S,P})$.
    \item 
    how the starting state $S$ --and corresponding formula $\mS(\mX)$-- and patterns $\pattern_g$ and $\pattern_h$ are computed.
\end{enumerate} 
Indeed, we considered several combinations (some of which presented in the ablation studies), and here we present four specific configurations of $\pattyd$, dubbed cautious, brave, reckless, and greedy,  
\begin{enumerate}
    \item sharing how the model of $\Pi^\pattern_{S,P}$ is computed, and 
    \item differing on how the starting state $S$ and/or patterns $\pattern_g$ and/or $\pattern_h$ are defined at each iteration.
\end{enumerate}
Before presenting the cautious, brave, reckless and greedy  implementations of $\pattyd$ in Subsection~\ref{subsec:cautious-brave-reckless-greedy}, we first describe  
\begin{enumerate}
    \item the goal value function used to select the next intermediate state $P'$ closer than $P$ to a goal state, if any (Subsection~\ref{subsec:goal-value-function}),
    \item the functions used to compute the pattern $\pattern_g$ allowing to reach $P$ from the initial state $I$, given the already computed pattern $\pattern_P$ and assignment 
    $\mu_P$ to the variables in $\pattern_P$ (Subsection~\ref{subsec:patterng}), and
    \item the functions used to compute the pattern $\pattern_h$ in state $P$, hopefully allowing to reach a state $P'$ closer than $P$ to a goal state  (Subsection~\ref{subsec:patternh}).
\end{enumerate}
We end the section presenting the different behaviours of the four procedures on the motivating example (Subsection~\ref{subsec:proc-motivating-example}).

\subsection{Computing the Goal Value Function}\label{subsec:goal-value-function}

Assume the last computed intermediate state $P$ satisfies a proper subset $G_P$ of the set $G$ of goals and that it has an associated goal value $P(\mGF_P(\mX)) = c > 0$. Then, for any state $P'$, the goal value $P'({\mGF}_P(\mX))$ which has to be associated to $P'$, 
\begin{enumerate}
    \item has to be zero when $P'$ is a goal state, 
    \item otherwise, if $P'$ is considered ``closer" than $P$ to a goal state, it has to be in $[0,\max(c - \varepsilon, 0)]$ in order to satisfy the constraint 
    ${\mGF}_P(\mX')  \le \max(P({\mGF}_P(\mX)) - \varepsilon, 0)$ in (\ref{eq:pi-s-pattern}),
    \item otherwise, it has to be $> \max(c - \varepsilon, 0)$.
\end{enumerate}
Specifically, in our case we consider a state $P'$  {\sl closer} than $P$ to a goal state if and only if 
\begin{enumerate}
    \item $P'$ satisfies a strict superset of the goals satisfied by $P$, or
    \item $G$ consists of a single numeric goal $\psi \unrhd 0$ and $P'(\psi) \ge P(\psi) + \varepsilon$.
\end{enumerate}

In the presence of multiple goals, a goal value function $\mGF_P(\mX)$ allowing to characterize such closer states is:
\begin{equation}\label{eq:gf-numgoal}
{\mGF}_P(\mX) = |G \setminus G_P| \times \sum_{g \in G_P} \ite(g,0,1) + \sum_{g \in G \setminus G_P} \ite(g,0,1).
\end{equation}
Indeed, 
\begin{enumerate}
    \item if $P'$ is a goal state, $P'({\mGF}_P(\mX)) = 0$,
    \item if $P'$ satisfies a strict superset of $G_P$,  then $P'({\mGF}_P(\mX)) \le \max(P({\mGF}_P(\mX)) - \varepsilon,0)$ since
    $P'({\mGF}_P(\mX)) \le |G \setminus G_P| - 1 < |G \setminus G_P| - \varepsilon =  P({\mGF}_P(\mX)) - \varepsilon$, and clearly $\max(P({\mGF}_P(\mX)) - \varepsilon,0) = P({\mGF}_P(\mX)) - \varepsilon$,
    \item if $P'$ does not satisfy a strict superset of $G_P$,  $P'({\mGF}_P(\mX)) \ge |G \setminus G_P| =  P({\mGF}_P(\mX))$ and thus $P'({\mGF}_P(\mX)) > \max(P({\mGF}_P(\mX)) - \varepsilon,0)$.
\end{enumerate}

Conversely, --similar to the Manhattan Distance Heuristic proposed by \citet{DBLP:conf/socs/ChenT24}-- if $G$ consists of a single numeric goal $g = \psi \unrhd 0$, 
a goal value function $\mGF_P(\mX)$ allowing to characterize the  states closer to satisfying $G = \set{\psi\unrhd 0}$ is:
\begin{equation}
    \label{eq:gf-psi}
{\mGF}_P(\mX) = \ite(\psi \unrhd 0, 0, -\psi + \varepsilon).%
\footnote{We return the term $-\psi + \varepsilon$ instead of the simpler $-\psi$ to ensure that the returned value is $> 0$ also when $G = \set{\psi > 0}$ and we consider a non-goal state $P'$ with $P'(\psi) = 0$.}
\end{equation}
Indeed, 
\begin{enumerate}
    \item if $P'$ is a goal state, i.e., if $P'(\psi) \unrhd 0$, $P'({\mGF}_P(\mX)) = 0$,
    \item otherwise, if $P'(\psi) \ge P(\psi) + \varepsilon$ then
    $P'({\mGF}_P(\mX)) \le \max(P({\mGF}_P(\mX)) - \varepsilon,0)$
    since $P'({\mGF}_P(\mX)) = -P'(\psi) + \varepsilon \le -P(\psi) = P({\mGF}_P(\mX)) - \varepsilon = \max(P({\mGF}_P(\mX)) - \varepsilon,0)$ (the last equality holds since $P'(\psi) \le 0$ and thus $- P(\psi) \ge \varepsilon \ge 0$) ,
    \item otherwise, $P'(\psi) < P(\psi) + \varepsilon$, i.e., $-P'(\psi) > -P(\psi) - \varepsilon$, and thus
     $P'({\mGF}_P(\mX)) = -P'(\psi) + \varepsilon > -P(\psi) = \max(P({\mGF}_P(\mX)) - \varepsilon,0)$.
\end{enumerate}

Notice that in (\ref{eq:gf-numgoal}) the first addend is the one which enforces that the goal in $G_P$ have to be satisfied in $P'$, if
$P'$ has to be closer than $P$ to a goal state. Of course, a simpler and likely more effective way to obtain the same result is to simply add the formula $\bigwedge_{g \in G_P} g$ as conjunct to $\Pi^\pattern_{S,P}$. In the implementation section, we will detail how the features of the \smt solver may be exploited to enhance overall effectiveness.

\subsection{Computing the Pattern to Reach the Intermediate State}\label{subsec:patterng}


In $\pattyd(\Pi)$, at each iteration we may select whether to start the search from the state $S=I$ or $S=P$. In both cases, the pattern $\pattern_g$ is meant to make the state $P$ reachable from $S$, but while any pattern will do when $S=P$ (and in such a case we select the empty pattern), when $S=I$ we refine the pattern $\pattern_P$ by discarding those actions in it which are not necessary to reach the state $P$ from $I$. Given the assignment $\mu_P$ to the variables in $\pattern_P = a_1; a_2; \ldots; a_k$, by construction, the plan
\begin{equation}\label{eq:plan-pi}
\pi=a_1^{\mu_P(a_1)}; a_2^{\mu_P(a_2)}; \ldots; a_k^{\mu_P(a_k)}
\end{equation}
is such that $\res(\pi,I) = P$. Thus, a simple solution is to exploit the pattern of $\pi$, simply obtained by removing from $\pattern_P$ the actions $a_i$ with $\mu_P(a_i) = 0$ and then recursively substitute each two consecutive occurrences of the same action with a single one. Such a pattern can be computed in linear time and is a subsequence of $\pattern_P$. 

\subsection{Computing the Pattern to Progress Beyond the Intermediate State}\label{subsec:patternh}


\begin{algorithm}[t]
\caption{Incomplete \arpg algorithm. \\
Input: a numeric planning problem~$\Pi= \langle X, A, I,G\rangle$, a state $S$, and a parameter $p \in \natural$. \\
Output: An incomplete \arpg, i.e., a sequence of disjunct set of actions. \label{alg:incomplete-arpg}
}
\newcommand{\iarpg}{\textsc{iarpg}}
\begin{algorithmic}[1]
\Function{$\textsc{computeIncompletePatternARPG}$}{$S, A, G, p$}
    \State $\iarpg \gets \epsilon$
    \State $A_1;A_2;\dots;A_k \gets \Call{computeARPG}{S,A,G}$
    \State $\hat{S}_0 \gets \Call{Relax}{S}$
    \State $\hat{S}_0;\hat{S}_1;\ldots;\hat{S}_k \gets \Call{getRelaxedStates}{A_1;A_2;\ldots;A_k, \hat{S}_0}$
    \State $i \gets k$
    \State $G' \gets G$
    \While{($i > 0$)}
        \State $i \gets i - 1$
        \State $\Gamma \gets \set{g \mid g \in G', \hat{S}_{i+1}~\mbox{satisfies}~g} \setminus \set{g \mid g \in G', \hat{S}_{i}~\mbox{satisfies}~g}$\label{alg:iarpg-Gamma}
        \If{($\Gamma = \emptyset$)} 
            \State \textbf{continue}
        \EndIf
        \State $A_\Gamma \gets \emptyset$
        \For{($g \in \Gamma$)}\label{alg:iarpg-for}
            \State $p_g \gets 0$
            \For{($a \in A_i$)}
                \If{(($res(a, \hat{S}_i)~\mbox{satisfies}~g)$ \textbf{and} ($p_g < p$))}
                    \State $A_\Gamma \gets A_\Gamma \cup \set{a}$
                    \State $p_g \gets p_g + 1$ \label{alg:iarpg-pg}
                \EndIf
            \EndFor
        \EndFor
        \State $G' \gets \set{g \mid g \in G', \hat{S}_{i}~\mbox{satisfies}~g} \cup \bigcup_{a \in A_\Gamma} \pre(a)$ \label{alg:iarpg-G'}
        \State $\iarpg \gets A_\Gamma;\iarpg$
        
    \EndWhile
    \State \Return $\textsc{computePattern}(\iarpg)$
\EndFunction
\end{algorithmic}
\end{algorithm}

In Subsection~\ref{subsec:arpg} we presented the procedure used by \citet{DBLP:conf/aaai/CardelliniGM24} to compute a pattern using the \arpg construction starting from $P$. The returned pattern is ensured to include all the actions executable in some state reachable from $P$. While this might be a desirable property (e.g., to have a complete pattern), in some cases we may adopt a greedy approach in which only a subset of the actions are considered, still ensuring that the goal conditions are satisfied in the relaxed state obtained after executing the actions in the pattern, of course starting from $P$.

In more details, let $A_1; A_2; \ldots; A_k$ be an \arpg, $\hat{S}_0$ be the relaxed state corresponding to $P$,  $\hat{S}_0;\hat{S}_1;\dots;\hat{S}_k$ be the sequence of induced relaxed states by the \arpg from $\hat{S}_0$. 
If completeness of the computed pattern is not required, we can determine an incomplete pattern ensuring that each goal $g$ becomes satisfied in some relaxed state, by setting an upper bound $p \ge 1$ on the number of actions whose execution in the relaxed state $\hat{S}_i$ causes the goal to become satisfied in $\hat{S}_{i+1}$. This is the idea in Algorithm~\ref{alg:incomplete-arpg}, whose  $\textsc{computeIncompletePatternARPG}(S,A,G,p)$ function returns a possibly incomplete pattern, while still ensuring that in the last induced relaxed state  $\hat{S}_k$ all the goals are satisfied. 
In the algorithm,
\begin{enumerate}
    \item $\Call{computeARPG}{S,A,G}$ returns the $\arpg$ computed as in Algorithm~\ref{alg:arpg-complete},
    \item $\Call{getRelaxedStates}{A_1;A_2;\ldots;A_k,\hat{S}_0}$ computes the sequence of relaxed states induced by $A_1;A_2;\ldots;A_k$ in $\hat{S}_0$.
\end{enumerate}
The algorithm proceeds backward: it starts with $i = k-1$ in the first iteration and decrements $i$ at each step until $i > 0$. In each iteration, 
\begin{enumerate}
    \item the subset $\Gamma$ of goals that are relaxed satisfied in $\hat{S}_{i+1}$ and not in $\hat{S}_{i}$ is computed (Line~\ref{alg:iarpg-Gamma}), 
    \item for each $g \in \Gamma$, at most $p$ actions leading to a relaxed state satisfying $g$ are selected and returned (Lines~\ref{alg:iarpg-for}-\ref{alg:iarpg-pg}),
    \item the preconditions of the selected actions are added to the set of goal conditions that have to be relaxed satisfied in $\hat{S}_{i}$ (Line~\ref{alg:iarpg-G'}).
\end{enumerate}
In our motivating example, if $G$ is (\ref{eq:G3}),  $\textsc{computeIncompletePatternARPG}(I,A,G, p)$ with $p=1$  produces the  \arpg in (\ref{eq:inc-arpg}) and the corresponding incomplete pattern~(\ref{eq:inc-pattern-G1-plan}).

\subsection{Cautious, Brave, Reckless and Greedy Procedures}\label{subsec:cautious-brave-reckless-greedy}

Cautious, brave, reckless and greedy procedures correspond to specific implementations of the procedures setting the starting state $S$ and the patterns $\pattern_g$ and $\pattern_h$ at each iteration. 
To help the reader in following the differences between the procedures, in Figures \ref{fig:infographic:cautious}, \ref{fig:infographic:brave}, \ref{fig:infographic:reckless}, and \ref{fig:infographic:greedy} of Appendix \ref{appendix:infographic} we show the infographics of the cautious, brave, reckless and greedy procedures, respectively. We remind that all the four procedures use the same goal value function to select the next intermediate state $P'$, that $P$ is the last computed intermediate state, and that $\pattern = \pattern_g;\pattern_h$.
More in details, in \pattyd, the pair $S$, $\pattern_g$ and $\pattern_h$
are respectively defined at each iteration by either $\stp(I,P,\mu_P,\pattern_P)$ and $\ptg(P,A,G)$ or 
$\stpn(I,P,\mu_P,\pattern_P,n)$ and $\ptgn(S,P,A,G,n)$. 
Notice that the former (resp. latter) pair of functions is invoked only if  $\Pi^{\pattern}_{S,P}$ is satisfiable (resp. unsatisfiable), i.e., only if, in the considered iteration, we {\em succeeded} (resp. {\em failed}) in finding a new intermediate state. Then, we formally define the following four procedures:  
\begin{enumerate}
    \item {\sl Cautious $\patty_D$} (Fig. \ref{fig:infographic:cautious} in the appendix): The pair $S, \pattern_g$ returned by the procedures $\stp(I, P, \mu_P, \pattern_P)$ and $\stpn(I, P, \mu_P, \pattern_P, n)$ has $S = I$, and $\pattern_g$  obtained by removing the actions in $\pattern_P$ that are not necessary to reach $P$, following the procedure defined in Subsection~\ref{subsec:patterng}. The pattern $\pattern_h$ returned by $\ptg(P, A, G)$ is computed invoking  $\textsc{computePatternARPG}(P,A,G)$ algorithm presented in Subsection~\ref{subsec:arpg}. The  function $\ptgn(P, A, G, n)$ returns the pattern obtained by adding a complete pattern at the end of the 
    pattern $\pattern_h$ used in the previous iteration. More in details, the added pattern is the one computed by $\ptg(P, A, G)$, once modified to ensure its completeness, i.e., after replacing Line~\ref{arpg2pattern} in Algorithm~\ref{alg:arpg-complete} with
    \begin{equation}\label{eq:arpg2pattern-mod}
                \Return\ \ \  \textsc{computePattern($\arpg; \aleft$)}.
    \end{equation}
    Thus, after $k>0$ consecutive failures of determining a state $P'$ closer than $P$ to a goal state, the procedure $\ptgn(P, A, G, n)$ returns a  $(k+1)$-complete pattern $\pattern_h$ and sets $S=I$.
    
    \item {\sl Brave $\patty_D$} (Fig. \ref{fig:infographic:brave} in the appendix): Differently from the cautious procedure, now the pair $S, \pattern_g$ returned by  $\stp(I, P, \mu_P, \pattern_P)$ has $S=P$ and $\pattern_g = \epsilon$. The other three procedures
    $\stpn(I, P, \mu_P, \pattern_P, n)$, $\ptg(P,A,G)$ and $\ptgn(S,P,A,G,n)$ are the same as in the cautious case.
    Thus, after $k>0$ consecutive failures of determining a state $P'$ closer than $P$ to a goal state, $\ptgn(P, A, G, n)$ returns a  $(k+1)$-complete pattern $\pattern_h$ and sets $S=I$.
    
    \item {\sl Reckless $\patty_D$} (Fig. \ref{fig:infographic:reckless} in the appendix): The pair $S, \pattern_g$ returned by the procedure  $\stp(I, P, \mu_P, \pattern_P)$ has $S = P$, and $\pattern_g = \epsilon$, and the same holds also for $\stpn(I, P, \mu_P, \pattern_P, n)$. $\ptg(P,A,G)$ is the same one used in the cautious and brave cases.
    Differently from the previous two cases, since the search always starts from $P$ even in the case of (repeated) failures, $\ptgn(S,P,A,G,n)$ adds a  pattern at the end of the 
    $\pattern_h$ used in the previous iteration, this time  the pattern being added at each iteration is the one computed by $\ptg(P, A, G)$ (without the modification in (\ref{eq:arpg2pattern-mod})). Assuming the pattern computed by $\ptg(P, A, G)$ is complete, after $k>0$ consecutive failures, $\ptgn(P, A, G, n)$ returns a  $(k+1)$-complete~$\pattern_h$ and sets $S=P$.
    
    \item {\sl Greedy $\patty_D$} (Fig. \ref{fig:infographic:greedy} in the appendix): The pair $S, \pattern_g$ returned by the procedure $\stp(I, P, \mu_P, \pattern_P)$ has $S = P$ and $\pattern_g = \epsilon$, and the same holds also for $\stpn(I, P, \mu_P, \pattern_P, n)$, as in the reckless case. $\ptg(P, A, G)$ returns the pattern computed by \Call{$\textsc{computeIncompletePatternARPG}$}{$P, A, G, p$} with $p=1$, which is (possibly) incomplete .
    The procedure $\ptgn(P, A, G, n)$ uses the same mechanism as $\ptg(P, A, G)$ to compute $\pattern_h$, but with $p = 2, 4, 8, \ldots$ actions at the first, second, third, \ldots consecutive failure, until $\pattern_h$ is no longer extended. At that point, it continues to extend $\pattern_h$ at each failure by concatenating the pattern computed by $\ptg(P, A, G)$, as in the reckless case. Thus,  assuming the pattern computed by $\ptg(P, A, G)$ is complete, the function $\ptgn(P, A, G, n)$ returns a  $k$-complete pattern $\pattern_h$ after at most $(k + \lceil \log_2(|A|)\rceil$ consecutive calls.
\end{enumerate}

\begin{theorem}\label{th:cautious-brave} 
Let $\Pi$ be a numeric planning problem. Cautious, brave, reckless and greedy $\pattyd(\Pi)$ procedures are correct. The cautious and brave procedures are also complete.
\end{theorem}

\begin{proof} 
    Correctness follows from the correctness of the encoding. For completeness, assume there exists a valid plan of length $n$. Consider the procedure $\pattyd(\Pi)$ in Algorithm~\ref{alg:pattyd} with $\mGF_P(\mX)$, $\stp(I,P,\mu_p,\pattern_P)$, $\stpn(I,P,\mu_p,\pattern_P,n)$ defined according to the cautious and brave specifications of $\pattyd(\Pi)$.
    
    Each time Line~\ref{alg:d:h-n-pattern} is executed, a complete pattern $\pattern_h$ is concatenated to the previously computed $\pattern_h$, while each time Line~\ref{alg:d:h-pattern} is executed the goal value of the newly computed state $P$ decreases at least by $\varepsilon > 0$. Thus, Line~\ref{alg:d:h-pattern} can be executed at most $\lceil I(\mGF_I(\mX))/\varepsilon\rceil$ times, each time after at most after $(n-1)$ iterations between two consecutive executions (since after $n$ consecutive executions of Line~\ref{alg:d:h-n-pattern}, $\pattern_h$ would be $n$-complete and all the goals in $G$ would be satisfiable). Thus, a valid plan will be found at most at the $(p+n)$-th iteration, where in the worst case $p = \lceil I(\mGF_I(\mX))/\varepsilon\rceil \times (n-1)$.
\end{proof}

Consider the cautious, brave, reckless, and greedy procedures, and assume that, at some iteration --not necessarily the same for all-- each has identified $P$ as an intermediate state, reachable from the initial state $I$ via a pattern $\pattern_g$. As already said, the  four procedures differ in how they define the formula $\Pi^\pattern_{S,P}$ to be used in the subsequent iterations, until a new state $P'$ is found that is closer to the goal than $P$.
If we consider the $i$-th iteration ($i \ge 1$) following the one in which $P$ was computed, then, from the definitions:
\begin{enumerate}
    \item For the greedy procedure, $\Pi^\pattern_{S,P} = \Pi^{\pattern_i}_{P,P}$, where $\pattern_i$ is a subpattern of the pattern $\pattern_h$ used by the cautious, brave, and reckless procedures.
    
    \item For the reckless procedure, $\Pi^\pattern_{S,P} = \Pi^{\pattern_h}_{P,P}$.
    
    \item For the brave procedure, $\Pi^\pattern_{S,P} = \Pi^{\pattern_h}_{P,P}$ if $i = 1$, and $\Pi^\pattern_{S,P} = \Pi^{\pattern_g;\pattern_h}_{I,P}$ otherwise.
    
    \item For the cautious procedure, $\Pi^\pattern_{S,P} = \Pi^{\pattern_g;\pattern_h}_{I,P}$.
\end{enumerate}
Further, if
\begin{enumerate}
    \item $g/r/b/c$ denote the number of iterations required by the greedy/reckless/brave/cautious procedures respectively, before finding a model of $\Pi^\pattern_{S,P}$; and
    \item since any model of $\Pi^{\pattern_i}_{P,P}$ can be extended to a model of $\Pi^{\pattern_h}_{P,P}$, and any model of $\Pi^{\pattern_h}_{P,P}$ can be extended to a model of $\Pi^{\pattern_g;\pattern_h}_{I,P}$,
\end{enumerate}
then $g \ge r \ge b \ge c$, i.e., the greedy (resp. cautious) procedure requires the most (resp. fewest) iterations before successfully finding a state $P'$ closer to the goal than $P$.
On the other hand, since $\pattern_i$ contains at most the same number of actions as $\pattern_h$, and $\pattern_h$ contains at most the same number of actions as $\pattern_g;\pattern_h$, we can expect each procedure to be able to determine the state $P'$ with fewer iterations and/or using formulas $\Pi^\pattern_{S,P}$ with fewer action variables, compared to each of the other three. Indeed, it can be expected that the higher the number of action variables in $\Pi^\pattern_{S,P}$, the more difficult it is to determine its satisfiability. For this reason, we can  expect each procedure to perform comparatively better than the others, at least on some problems or domains.

\subsection{Cautious, Brave, Reckless and Greedy Procedures on the Motivating Example}\label{subsec:proc-motivating-example}

Consider the problem $\Pi$ in the motivating example, in which initially the conditions in (\ref{eq:ex-init}) are satisfied. We now show that it is indeed the case that, depending on the goal, each procedure can find a valid plan with fewer iterations and/or using formulas $\Pi^\pattern_{S,P}$ with fewer action variables, compared to the other three.

We start with the simple goals to have at least 5 objects and 50 tokens, i.e., with $G$ as in~(\ref{eq:G3}). Then,
greedy/reckless/ brave/cautious $\pattyd(\Pi)$ find a valid plan with just one iteration, but
\begin{enumerate}
    \item greedy $\pattyd(\Pi)$ using the pattern $\pattern=\buy(1)$ and thus with 1 action variable in the corresponding formula $\Pi^\pattern_{I,I}$, while
    \item reckless/brave/cautious using the pattern $\pattern$ as in (\ref{eq:ex-pattern}), and thus with $(2\times m + 2)$ action variables in the corresponding formula $\Pi^\pattern_{I,I}$.
\end{enumerate}

Assume now the goal is to have at least 5 objects and 500 tokens, i.e., with $G$ as in~(\ref{eq:G1}). Then,
greedy/reckless/ brave/cautious $\pattyd(\Pi)$ with 2/1/1/1 iterations, since
\begin{enumerate}
    \item greedy $\pattyd(\Pi)$ will first try using the pattern $\pattern=\buy(1);\sell(1)$ and thus switch to the complete pattern (\ref{eq:ex-pattern}) used by the reckless/brave/cautious in their first and only iteration, while
    \item reckless/brave/cautious will just use the pattern $\pattern$ as in (\ref{eq:ex-pattern}) as in the previous case.
\end{enumerate}

Assume the goal is to have at least 50 objects and 500 tokens, i.e., with $G$ as in~(\ref{eq:G2}). Then,
greedy/reckless/brave/ cautious $\pattyd(\Pi)$ will reach it with 4/2/2/2 iterations, since
\begin{enumerate}
    \item greedy $\pattyd(\Pi)$ will first try using the pattern $\pattern=\buy(1);\sell(1)$ and then switch to the complete pattern (\ref{eq:ex-pattern}) used by the reckless/brave/cautious in their first iteration to get to the intermediate state $P$ satisfying $\At(m)$, $\bought = 5$ and $\cash=500$,
    \item in state $P$, greedy $\pattyd(\Pi)$ will first try using the pattern $\pattern=\buy(1)$ and then switch to the complete pattern (\ref{eq:ex-pattern2}) used by the reckless/brave/cautious in their second iteration to reach a goal state.
\end{enumerate}
The reckless, brave, and cautious procedures use the same formula in their first iteration. In the second iteration, however, both the reckless and brave procedures use $P$ as the starting state, while the cautious procedure, unnecessarily, resets the starting state to $I$. As a result, after refining the pattern used in the first iteration, the formula $\Pi^\pattern_{S,P}$ used by the cautious procedure  will contain $(m + 1 + 2 \times m + 2)$ action variables, more than the $(2 \times m + 2)$ action variables in the  formula used by the brave and reckless procedures.
    
Finally, assume the goal is to have at least 45 objects and 550 tokens, i.e.,  
$$
G = \set{\bought \ge 45, \cash = 550}.
$$
Then, reckless/brave/cautious $\pattyd(\Pi)$ will all first use the complete pattern (\ref{eq:ex-pattern}) to get to the intermediate state $P$ satisfying  $\At(m)$, $\bought = 0$ and $\cash=550$, obtained executing the plan
    $$
    \buy(1)^5;\travel(1,2);\travel(2,3); \ldots; \travel(m-1;m); \sell(m)^5.
    $$
In state $P$, the complete pattern (\ref{eq:ex-pattern2}) will be used by the reckless/brave/cautious, but only the cautious procedure can find the valid plan
    \begin{gather*}
    \buy(1)^{10};\travel(1,2);\travel(2,3); \ldots; \travel(m-1;m); \\
    \sell(m)^{10}; 
    \travel(m,m-1);\travel(m-1,m-2); \ldots; \travel(2;1); \buy(1)^{45}.        
    \end{gather*}
    at the second iteration.
    Indeed, only the cautious procedure will fruitfully start the second search from $I$, and thus will be able to correctly decide that 10 (and not 5) objects needs to be initially bought in the first market. The brave procedure will need a third iteration to set the starting state to $I$, and it will also consider the unnecessarily long pattern obtained by concatenating (\ref{eq:ex-pattern2}) once more. The reckless procedure, given its early commitment to buy 5 objects in the initial state and then move to the last market to sell them,  when starting from state $P$ needs $m$ iterations --each time chaining pattern (\ref{eq:ex-pattern2})-- to reach the goal state satisfying $\At(m), \bought=50, \cash=550$ (corresponding to the plan in which 
    the agent travels back to the first market, there buys 50 objects and then goes again in the last market to sell 5 of them). 
    The greedy procedure explores the search space analogously to the reckless procedure, except for additional calls (e.g., the initial call with $S = P = I$ and $\pattern = \buy(1);\sell(1)$) introduced by first trying some incomplete pattern.


\section{Implementation and Experimental Results}\label{sec:impl-exper}

Consider a planning problem $\Pi = \tuple{X,I,A,G}$. To have a goal directed behaviour of $\pattyd(\Pi)$, and, more in general, reduce the number of iterations and obtain  best overall performance, at each iteration (corresponding to a given value for the variables $S, P$ and $\pattern$), we would like $\textsc{Solve}(\Pi^{ \pattern}_{S,P})$ to return an assignment leading to a state $P'$ with is as close as possible to a goal state. Naturally, how to compute such state effectively $P'$ critically depends on the functionalities and performance of the underlying \smt{} solver. In our case, we use \texttt{Z3} version 4.12.2 --upgrading from version \texttt{v4.8.7} used by \citet{cardellini2024patterns}-- \cite{de2008z3}, which supports the specification of hard constraints expressed as assertions, soft constraints, and objective functions for minimization or maximization. These features, assuming the intermediate state $P$ already satisfied a subset $G_P \subseteq G$ of the goals, enable us to implement the following mechanism for selecting the state $P'$:
\begin{enumerate}
\item The formulas $\mS(\mX)$ and $\mT^\pattern(\mX, \mA^\pattern, \mX')$ are added as hard constraints (assertions), ensuring that the extracted plan has to be executable starting from $S$.
\item If $G$ consists of a single numeric goal $g = \psi \unrhd 0$, we exploit the goal value function $\mGF_P(\mX)$ defined as in (\ref{eq:gf-numgoal}) by adding
\begin{equation}
    \label{eq:GF-spec}
{\mGF}_P(\mX')  \le \max(P({\mGF}_P(\mX)) - \varepsilon, 0) 
\end{equation}
as hard constraint, and introducing
$\mGF_P(\mX')$ as an objective function that the solver has to minimize.
\item 
If $G$ does not consist of a single numeric condition, the additional assertion
$$    
\bigwedge_{g \in G_P} g(\mX') \;\wedge\; \bigvee_{g \in G \setminus G_P} g(\mX')
$$
is included to require that the resulting state $P'$ both preserves the goals already satisfied in $P$ and satisfies at least one additional goal in $G \setminus G_P$. Moreover, each goal in the set $G \setminus G_P$ of goals yet to be satisfied is added as a soft constraint, thereby instructing \texttt{z3} to attempt to satisfy as many of them as possible. With such additions, it is not necessary to add (\ref{eq:GF-spec}) as assertion with $\mGF_P(\mX)$ defined as in (\ref{eq:gf-numgoal}).
\end{enumerate}
$\pattyd(\Pi)$ is thus goal-directed and surely will return a valid plan whenever $\Pi$ has a valid plan of length $k$, a $k$-complete pattern $\pattern = \pattern_g;\pattern_h$ is generated, and the starting state $S$ is the initial state $I$. 

For the experiments, we considered the settings and 20 domains and 20 problems per domain used  in the Agile track of the last 2023 Numeric \ipc, and also the 20 problems in the line exchange domain introduced by \citet{DBLP:conf/aaai/CardelliniGM24}. Thus, each system had a time limit of $5$ minutes on an Intel Xeon Platinum $8000$ $3.1$GHz with 8 GB of RAM. 

\subsection{Cautious, Brave, Reckless and Greedy \pattyd vs \pattyo}

            \begin{table*}[tb]
            \centering
            \resizebox{\textwidth}{!}{
\begin{tabular}{|l||ccccc||ccccc||ccccc||ccccc||}
\hline
 & \multicolumn{5}{c||}{Solved (out of $20$)}&\multicolumn{5}{c||}{Time (s)}&\multicolumn{5}{c||}{\textsc{smt} calls}&\multicolumn{5}{c||}{$|\mathcal{A}^\prec|$}\\
Domain & $\mathrm{P}_\textsc{o}$&$\mathrm{P}_\textsc{c}$&$\mathrm{P}_\textsc{b}$&$\mathrm{P}_\textsc{r}$&$\mathrm{P}_\textsc{g}$&$\mathrm{P}_\textsc{o}$&$\mathrm{P}_\textsc{c}$&$\mathrm{P}_\textsc{b}$&$\mathrm{P}_\textsc{r}$&$\mathrm{P}_\textsc{g}$&$\mathrm{P}_\textsc{o}$&$\mathrm{P}_\textsc{c}$&$\mathrm{P}_\textsc{b}$&$\mathrm{P}_\textsc{r}$&$\mathrm{P}_\textsc{g}$&$\mathrm{P}_\textsc{o}$&$\mathrm{P}_\textsc{c}$&$\mathrm{P}_\textsc{b}$&$\mathrm{P}_\textsc{r}$&$\mathrm{P}_\textsc{g}$\\
\hline
\textsc{BlGrp} (S)&\textbf{20}&\textbf{20}&\textbf{20}&\textbf{20}&\textbf{20}&1.8&1.9&2.0&2.0&\textbf{1.7}&\textbf{1.0}&\textbf{1.0}&\textbf{1.0}&\textbf{1.0}&1.4&87&87&87&87&\textbf{26}\\
\textsc{Cnt} (S)&\textbf{20}&\textbf{20}&\textbf{20}&\textbf{20}&\textbf{20}&0.9&\textbf{0.8}&0.9&0.9&1.1&\textbf{1.0}&\textbf{1.0}&\textbf{1.0}&\textbf{1.0}&4.3&\textbf{45}&\textbf{45}&\textbf{45}&\textbf{45}&80\\
\textsc{Cnt} (L)&\textbf{20}&\textbf{20}&\textbf{20}&\textbf{20}&\textbf{20}&1.0&\textbf{0.9}&\textbf{0.9}&\textbf{0.9}&1.0&\textbf{1.0}&\textbf{1.0}&\textbf{1.0}&\textbf{1.0}&\textbf{1.0}&\textbf{46}&\textbf{46}&\textbf{46}&\textbf{46}&\textbf{46}\\
\textsc{Del} (S)&5&8&8&8&\textbf{9}&227.3&187.0&189.5&189.0&\textbf{171.8}&\textbf{2.2}&\textbf{2.2}&\textbf{2.2}&\textbf{2.2}&5.0&297&312&297&297&\textbf{6}\\
\textsc{Drn} (S)&3&\textbf{16}&\textbf{16}&13&\textbf{16}&255.3&110.8&99.9&110.8&\textbf{98.2}&\textbf{5.7}&12.0&12.0&12.7&9.0&\textbf{12}&31&37&23&23\\
\textsc{Exp} (S)&\textbf{4}&\textbf{4}&\textbf{4}&2&2&\textbf{243.2}&246.1&249.2&270.8&270.8&\textbf{4.5}&5.5&5.5&5.5&\textbf{4.5}&\textbf{96}&444&444&444&444\\
\textsc{Farm} (S)&\textbf{20}&\textbf{20}&\textbf{20}&\textbf{20}&17&0.9&\textbf{0.8}&0.9&\textbf{0.8}&45.7&\textbf{1.0}&\textbf{1.0}&\textbf{1.0}&\textbf{1.0}&6.6&22&22&22&22&\textbf{11}\\
\textsc{Farm} (L)&\textbf{20}&\textbf{20}&\textbf{20}&\textbf{20}&17&4.3&3.3&3.6&\textbf{3.2}&48.0&\textbf{1.0}&\textbf{1.0}&\textbf{1.0}&\textbf{1.0}&4.9&28&28&28&28&\textbf{11}\\
\textsc{HPwr} (S)&\textbf{20}&\textbf{20}&\textbf{20}&\textbf{20}&14&\textbf{17.2}&97.1&90.1&93.4&230.2&\textbf{1.0}&\textbf{1.0}&\textbf{1.0}&\textbf{1.0}&3.0&\textbf{146}&151&\textbf{146}&\textbf{146}&292\\
\textsc{Mrkt} (L)&-&2&5&5&\textbf{19}&-&294.8&279.2&279.8&\textbf{19.5}&-&\textbf{3.0}&\textbf{3.0}&\textbf{3.0}&21.5&-&100&90&90&\textbf{3}\\
\textsc{MPrime} (S)&12&\textbf{14}&\textbf{14}&\textbf{14}&10&133.6&\textbf{119.0}&128.7&124.3&154.9&\textbf{1.1}&1.3&1.3&1.3&2.9&653&666&666&664&\textbf{554}\\
\textsc{PathM} (S)&\textbf{20}&\textbf{20}&\textbf{20}&\textbf{20}&\textbf{20}&4.7&4.4&\textbf{4.0}&4.7&4.8&\textbf{1.0}&\textbf{1.0}&\textbf{1.0}&\textbf{1.0}&\textbf{1.0}&353&353&353&353&\textbf{77}\\
\textsc{PlWat} (S)&6&\textbf{20}&\textbf{20}&19&\textbf{20}&215.5&20.8&16.9&23.9&\textbf{10.7}&\textbf{8.2}&14.5&14.5&15.0&19.8&\textbf{32}&158&142&59&64\\
\textsc{Rvr} (S)&13&13&\textbf{14}&10&7&129.9&114.9&\textbf{113.6}&158.8&200.7&\textbf{1.4}&\textbf{1.4}&\textbf{1.4}&\textbf{1.4}&7.2&\textbf{309}&317&\textbf{309}&\textbf{309}&480\\
\textsc{Sail} (S)&\textbf{20}&\textbf{20}&\textbf{20}&\textbf{20}&\textbf{20}&1.4&1.1&\textbf{0.9}&1.0&\textbf{0.9}&\textbf{3.3}&\textbf{3.3}&\textbf{3.3}&\textbf{3.3}&7.2&31&44&31&31&\textbf{28}\\
\textsc{Sail} (L)&\textbf{20}&\textbf{20}&\textbf{20}&\textbf{20}&\textbf{20}&4.5&1.4&\textbf{0.8}&0.9&0.9&\textbf{1.4}&\textbf{1.4}&\textbf{1.4}&\textbf{1.4}&3.9&34&38&34&34&\textbf{12}\\
\textsc{Stlrs} (S)&\textbf{15}&\textbf{15}&\textbf{15}&\textbf{15}&\textbf{15}&85.5&83.5&83.7&82.8&\textbf{79.1}&\textbf{1.0}&\textbf{1.0}&\textbf{1.0}&\textbf{1.0}&1.4&576&576&576&576&\textbf{176}\\
\textsc{Sgr} (S)&\textbf{20}&\textbf{20}&\textbf{20}&15&15&\textbf{6.0}&12.2&14.9&81.3&78.4&\textbf{2.4}&4.4&4.4&4.6&4.6&\textbf{233}&481&487&454&449\\
\textsc{Tpp} (L)&2&3&3&\textbf{4}&3&270.4&259.3&255.6&\textbf{252.5}&264.6&2.5&\textbf{2.0}&\textbf{2.0}&\textbf{2.0}&5.5&\textbf{54}&61&\textbf{54}&\textbf{54}&108\\
\textsc{Zeno} (S)&\textbf{11}&\textbf{11}&\textbf{11}&\textbf{11}&\textbf{11}&146.5&136.8&\textbf{136.7}&137.1&137.8&\textbf{1.6}&\textbf{1.6}&\textbf{1.6}&\textbf{1.6}&4.6&\textbf{256}&264&\textbf{256}&\textbf{256}&437\\
\textsc{Line} (L)&\textbf{20}&\textbf{20}&\textbf{20}&\textbf{20}&\textbf{20}&\textbf{1.1}&1.6&1.5&1.4&1.4&\textbf{2.9}&6.7&6.7&7.0&7.8&\textbf{23}&82&81&67&65
\\\hline
\textit{Best}&\textbf{291}&\textbf{326}&\textbf{330}&\textbf{316}&\textbf{315}&\textbf{52}&\textbf{55}&\textbf{59}&\textbf{54}&\textbf{138}&\textbf{290}&\textbf{256}&\textbf{247}&\textbf{247}&\textbf{116}&\textbf{160}&\textbf{81}&\textbf{117}&\textbf{140}&\textbf{208}\\\hline

        \end{tabular}}
        \caption{Cautious, Brave, Reckless and Greedy \pattyd vs \pattyo. Planner names are abbreviated.}
        \label{tab:cbrg}
        \end{table*}

We first considered the cautious, brave, reckless and greedy versions of $\pattyd$, hereby abbreviated respectively with  $\pattyc$, $\pattyb$ and $\pattyb$ and $\pattyg$.
To assess the improvement against the previous approach based on patterns by \citet{DBLP:conf/aaai/CardelliniGM24}, we also include the original version of \patty, called \pattyo, implementing the planning as satisfiability approach presented in subsection~\ref{subsec:planning-as-sat-patterns}. Notice that
\pattyo was shown to have better performance on every single domain than  the Planning as Satisfiability planners $(i)$
 \textsc{OMTPlan}~\cite{Leofante2020},  $(ii)$ the system \re{} by \citet{DBLP:conf/aaai/CardelliniGM24} implementing  the relaxed-relaxed-$\exists$ encoding proposed by \citet{DBLP:conf/ijcai/BofillEV17}, and $(iii)$ \textsc{Springroll}~\cite{Scala_Ramirez_Haslum_Thiebaux_2016_Rolling}. \pattyo is thus the  \textsc{sota} among the symbolic numeric planners.
 
 Table~\ref{tab:cbrg} presents the results.
In the first three subtables of the table, we show: the name of the domain (subtable Domain); the number of solved problems (subtable Solved); the average time to find a solution, counting the time limit when the solution could not be found (subtable Time). In the last three subtables, we considered only the problems solved by all the  planners
and show
the average number of calls to the \smt solver (subtable \smt calls); the number of action variables in the pattern (subtable $|A^\pattern|$)  when a solution is found. 
The last row (``\textit{Best}'') reports on how many of the 420 considered problems, each planner obtained the best result, representing respectively, whether the problem was solved, the lowest time, the lowest number of \smt calls and the lowest number of variables in 
$A^\pattern$. Thus, for example, \pattyo was able to solve 291 problems, was the fastest planner on 52, had the lowest number of \smt calls on 290 and had the lowest number of action variables in the last call to the \smt solver on 160 of the 291 problems it solved.

Looking at the performance results, the first observation is that the solvers in the cautious/brave pair and those in the reckless/greedy pair of $\pattyd(\Pi)$ exhibit overall comparable performance, and that all the four versions of $\pattyd(\Pi)$ perform better than \pattyo. A closer look however, highlights the different performances of the solvers in different domains. If for each domain we rank solvers first on the number of solved problems and then on time required,
\begin{enumerate}
    \item \pattyo performs best on 4 domains
    (\textsc{Exp(S)},  \textsc{HPwr(S)}, \textsc{Sgr(S)}, \textsc{Line(L)}) 
    and worst on 10 domains
    (\textsc{Cnt(L)}, 
    (\textsc{Del(S)}, 
    \textsc{Drn(S)}, 
    \textsc{Mrkt(L)},  
    \textsc{PlWat(S)}, 
    \textsc{Sail(S)}, 
    \textsc{Sail(L)}, 
    \textsc{Stlrs(S)}, 
    \textsc{Tpp(L)}), 
    \textsc{Zeno(S)}),

    \item \pattyc performs best on 4 domains
    (\textsc{Cnt(S)}, \textsc{Cnt(L)}, \textsc{Farm(S)}, \textsc{MPrime(S)})
    and worst on 1 domain 
    (\textsc{Line(L)}),
    
    \item \pattyb performs best on 4 domains
    (\textsc{Cnt(L)}, \textsc{Path(M)}, \textsc{Sail(S)}, \textsc{Sail(L)})
    and worst on 1 domain
    (\textsc{BlGrp(S)}),
    
    \item \pattyr performs best on 4 domains
    (\textsc{Cnt(L)}, \textsc{Farm(S)}, \textsc{Tpp(L)})
    and worst on 3 domains
    (\textsc{BlGrp(S)}),
    \textsc{Exp(S)}),
    (\textsc{Sgr(S)}))

    \item \pattyg performs best on 7 domains
    (\textsc{BlGrp(S)}, \textsc{Del(S)},\textsc{Drn(S)}, \textsc{Mrkt(L)}, \textsc{PlWat(S)},
    \textsc{Sail(S)}, \textsc{Stlrs(S)})
    and worst on 9 domains
    (\textsc{Cnt(S)}, \textsc{Cnt(L)}, \textsc{Exp(S)}, \textsc{Farm(S)},
    \textsc{Farm(L)}, \textsc{HPwr(S)},
    \textsc{MPrime(S)}, \textsc{PathM(S)}, \textsc{Rvr(S)}. 
\end{enumerate}
Thus, $\pattyg$ is the planner with the most extreme performances in many domains, with its high number of positive results offset by an even higher number of negative results in other domains.
\pattyc and \pattyb have more balanced performances, leading to the best overall results. 
The possibly significant differences between the performances of the planners are best highlighted by the \textsc{Mrkt (L)} domain, where $\pattyg$ solves almost four times more problems than the others. Problems in this domain are characterized by $(i)$ a single numeric goal, and $(ii)$ a high number of executable plans in any possible starting state. In such cases, considering complete patterns to progress the search causes \texttt{z3} to get lost while attempting to find the plan that minimizes the given objective function. This is reflected by the low number of variables in $\mA^\pattern$ (column $|\mA^\pattern|$) and the high number of \smt calls (column ``\smt calls'') for \pattyg: For these problems, it is preferable to repeatedly solve many ``small'' optimization problems rather than a few ``large'' ones.
The results in the $|\mA^\pattern|$ and \smt calls subtables are in line with expectations. At the generic iteration $i$ of the procedures. each exploiting a pattern $\pattern = \pattern_g;\pattern_h$, we can expect:
\begin{enumerate}
\item \pattyo to use a longer pattern $\pattern$ than \pattyc, since in \pattyc the pattern $\pattern_g$ has been refined,
\item \pattyc to use a longer pattern $\pattern$ than \pattyb, since in \pattyb the pattern $\pattern_g$ can be empty,
\item \pattyb to use a longer pattern $\pattern$ than \pattyr, since in \pattyr the pattern $\pattern_g$ is always empty,
\item \pattyr to use a longer pattern $\pattern$ than \pattyg, since in \pattyg the pattern $\pattern_h$ can be incomplete.
\end{enumerate}
From the above, we can expect, on most domains, a correspondingly lower number of \smt calls for \pattyo/\pattyc/ \pattyb/\pattyr/\pattyg. Of course, there can be exceptions, since each two systems may employ different patterns in each iteration, e.g., because they compute a different intermediate state.

Overall, $\pattyc/\pattyb/\pattyr/\pattyg$ outperform the previous \textsc{sota} symbolic planner $\pattyo$, solving at least 24 additional problems --39 in the case of $\pattyb$, which achieves the highest number of solved problems. Notably, $\pattyc$ and $\pattyb$ solve at least as many problems as $\pattyo$ in every domain, and outperform it in 8 out of the 11 domains whose problems were not all solved by \pattyo.

            \begin{table*}[tb]
            \centering
            \resizebox{\textwidth}{!}{
\begin{tabular}{|l||cccccccc||cccccccc||}
\hline
 & \multicolumn{8}{c||}{Solved (out of $20$)}&\multicolumn{8}{c||}{Time (s)}\\
Domain & $\mathrm{P}_\textsc{c}$&$\mathrm{P}_\textsc{b}$&$\mathrm{P}_\textsc{r}$&$\mathrm{P}_\textsc{g}$&$\mathrm{EN}_\mathrm{CT}$&$\mathrm{EN}$&$\mathrm{NFD}$&$\mathrm{FF}$&$\mathrm{P}_\textsc{c}$&$\mathrm{P}_\textsc{b}$&$\mathrm{P}_\textsc{r}$&$\mathrm{P}_\textsc{g}$&$\mathrm{EN}_\mathrm{CT}$&$\mathrm{EN}$&$\mathrm{NFD}$&$\mathrm{FF}$\\
\hline
\textsc{BlGrp} (S)&\textbf{20}&\textbf{20}&\textbf{20}&\textbf{20}&14&16&-&2&1.9&2.0&2.0&\textbf{1.7}&117.2&81.5&-&270.2\\
\textsc{Cnt} (S)&\textbf{20}&\textbf{20}&\textbf{20}&\textbf{20}&10&12&11&15&\textbf{0.8}&0.9&0.9&1.1&163.8&133.8&149.8&95.7\\
\textsc{Cnt} (L)&\textbf{20}&\textbf{20}&\textbf{20}&\textbf{20}&12&10&6&8&\textbf{0.9}&\textbf{0.9}&\textbf{0.9}&1.0&142.3&170.9&214.0&180.0\\
\textsc{Del} (S)&8&8&8&9&14&13&9&\textbf{18}&187.0&189.5&189.0&171.8&117.1&121.7&165.2&\textbf{41.2}\\
\textsc{Drn} (S)&16&16&13&16&\textbf{18}&16&16&2&110.8&99.9&110.8&98.2&\textbf{55.4}&62.9&66.0&268.4\\
\textsc{Exp} (S)&4&4&2&2&\textbf{6}&\textbf{6}&3&-&246.1&249.2&270.8&270.8&224.0&\textbf{212.3}&253.7&-\\
\textsc{Farm} (S)&\textbf{20}&\textbf{20}&\textbf{20}&17&\textbf{20}&\textbf{20}&15&9&\textbf{0.8}&0.9&\textbf{0.8}&45.7&1.8&0.9&85.3&188.1\\
\textsc{Farm} (L)&\textbf{20}&\textbf{20}&\textbf{20}&17&\textbf{20}&18&11&15&3.3&3.6&3.2&48.0&\textbf{2.5}&48.6&151.2&80.5\\
\textsc{HPwr} (S)&\textbf{20}&\textbf{20}&\textbf{20}&14&\textbf{20}&2&1&1&97.1&90.1&93.4&230.2&\textbf{4.6}&270.3&285.1&285.0\\
\textsc{Mrkt} (L)&2&5&5&19&\textbf{20}&4&-&-&294.8&279.2&279.8&\textbf{19.5}&35.0&259.3&-&-\\
\textsc{MPrime} (S)&14&14&14&10&\textbf{17}&\textbf{17}&14&\textbf{17}&119.0&128.7&124.3&154.9&74.6&68.1&127.2&\textbf{45.1}\\
\textsc{PathM} (S)&\textbf{20}&\textbf{20}&\textbf{20}&\textbf{20}&3&2&1&10&4.4&\textbf{4.0}&4.7&4.8&262.8&272.2&284.2&154.9\\
\textsc{PlWat} (S)&\textbf{20}&\textbf{20}&19&\textbf{20}&\textbf{20}&16&14&3&20.8&16.9&23.9&\textbf{10.7}&41.1&101.3&167.2&268.3\\
\textsc{Rvr} (S)&13&\textbf{14}&10&7&12&8&4&10&114.9&\textbf{113.6}&158.8&200.7&143.7&197.4&240.8&133.3\\
\textsc{Sail} (S)&\textbf{20}&\textbf{20}&\textbf{20}&\textbf{20}&\textbf{20}&\textbf{20}&10&1&1.1&\textbf{0.9}&1.0&\textbf{0.9}&5.0&2.0&150.3&285.0\\
\textsc{Sail} (L)&\textbf{20}&\textbf{20}&\textbf{20}&\textbf{20}&2&2&15&8&1.4&\textbf{0.8}&0.9&0.9&270.8&270.6&96.8&182.8\\
\textsc{Stlrs} (S)&\textbf{15}&\textbf{15}&\textbf{15}&\textbf{15}&2&1&-&4&83.5&83.7&82.8&\textbf{79.1}&279.0&288.6&-&243.8\\
\textsc{Sgr} (S)&\textbf{20}&\textbf{20}&15&15&11&8&4&13&\textbf{12.2}&14.9&81.3&78.4&144.5&182.5&245.7&122.5\\
\textsc{Tpp} (L)&3&3&4&3&\textbf{7}&3&2&2&259.3&255.6&252.5&264.6&\textbf{212.3}&255.2&270.0&266.7\\
\textsc{Zeno} (S)&11&11&11&11&17&\textbf{19}&9&11&136.8&136.7&137.1&137.8&89.5&\textbf{28.1}&172.5&135.0\\
\textsc{Line} (L)&\textbf{20}&\textbf{20}&\textbf{20}&\textbf{20}&11&9&6&6&1.6&1.5&\textbf{1.4}&\textbf{1.4}&149.5&175.4&235.0&211.6
\\\hline
\textit{Best}&\textbf{326}&\textbf{330}&\textbf{316}&\textbf{315}&\textbf{276}&\textbf{222}&\textbf{151}&\textbf{155}&\textbf{35}&\textbf{38}&\textbf{37}&\textbf{93}&\textbf{34}&\textbf{21}&\textbf{36}&\textbf{90}\\\hline

        \end{tabular}}
        \caption{Comparative analysis between \pattyc, \pattyb, \pattyr and \pattyg and other publicly available search-based numeric planners.}
        \label{tab:search}
        \end{table*}

\subsection{Cautious, Brave, Reckless and Greedy \pattyd vs SOTA Planners}

As a second step, we compared \pattyc/\pattyb/\pattyr/\pattyg with the search-based planners
 \ENHSP with the configurations \texttt{sat-hadd}~\cite{Scala_Haslum_Thiébaux_Ramirez_2020}, \texttt{sat-aibr} \cite{Scala_Haslum_Thiebaux_Ramirez_2016_AIBR} and \texttt{sat-hmrphj}~\cite{scala2020search} considering for each problem its best resulting configuration,
 the very recent extension of the \ENHSP planner  introduced by \citet{DBLP:conf/socs/ChenT24} (that we call \ENHSPCT), \textsc{MetricFF} (\FF) \cite{hoffmann2003metric}, and 
 \textsc{NLM-CutPlan} (\NLM) in the {\texttt{sat} configuration which won the Agile track of that last \ipc} \cite{Kuroiwa_Shleyfman_Piacentini_Castro_Beck_2022}.%
\footnote{
For \ENHSP, \ENHSPCT~and \NLM, we got in touch with the authors about which configuration to use for their planners. See \url{https://ipc2023-numeric.github.io/results/presentation.pdf} for \ipc results.}
Table~\ref{tab:search} shows for each planner the number of problem it solves and the average time it takes, the latter computed as before. 
As it can be seen, \pattyc/\pattyb/\pattyr/\pattyg perform overall better on many domains
than each search-based planner being able to solve the highest number of problems in 13/14/11/9 domains respectively, compared to the 10  by \ENHSPCT, 5 by \ENHSP, 0 by \NLM, and 2 by \FF. 

\begin{figure}[t]
    \centering
    \includegraphics[width=\columnwidth]{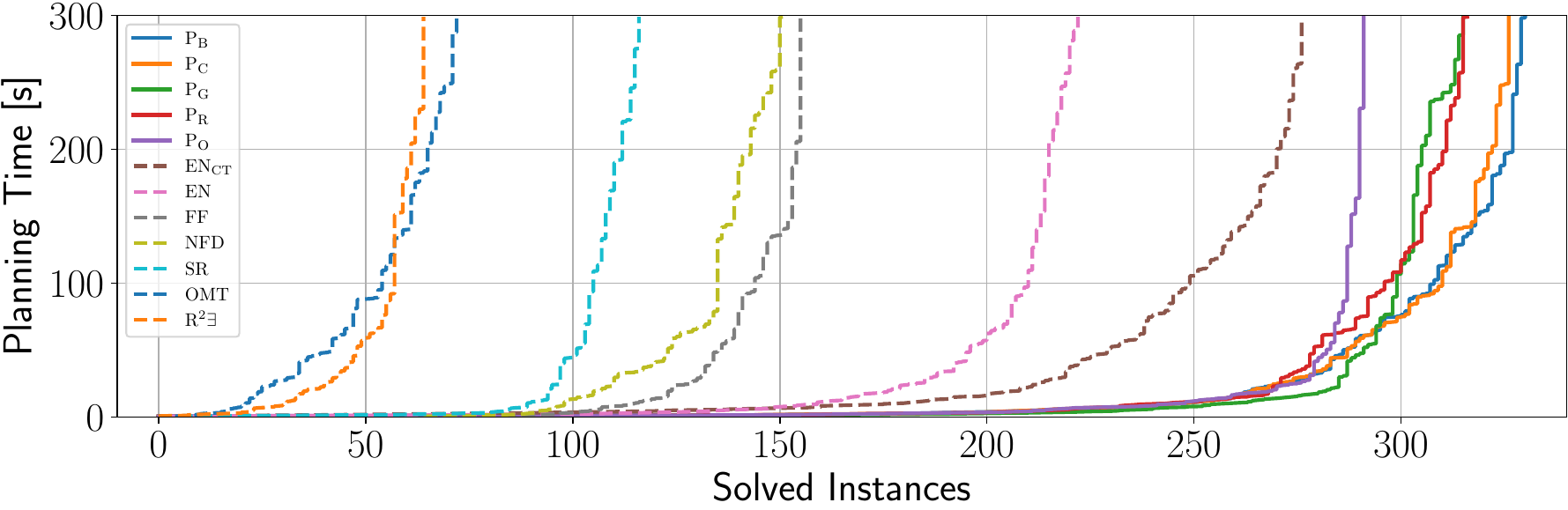}
    \caption{Number of problems solved ($x$-axis) within a given time ($y$-axis). Planner names are abbreviated. \spp planners are represented with  solid lines.}
    \label{fig:cactus-plot}
\end{figure}

The cactus plot in Figure~\ref{fig:cactus-plot} summarizes the performance of all the systems considered, showing how many problems each solves within a given time. As it can be seen, symbolic planners not based (resp. based) on patterns solve the fewest (resp. most). 
Clearly, different figures can be obtained by considering different domains/problems, especially if comparing symbolic vs search-based planners. Indeed, as also the above results point out,  depending on the domain, search-based planners may perform far better/worse than symbolic planners.
We also experimented with a 1800s time-limit, obtaining the same overall picture.


\section{Ablation Studies}\label{sec:ablation}

Consider a planning problem $\Pi$ and let $S$, $P$ and $\pattern = \pattern_g;\pattern_h$ the last computed starting state, intermediate state and pattern, respectively. To better understand what are the main reasons behind the positive performance of the various versions of \pattyd, we performed ablation studies varying
\begin{enumerate}
    \item the goal value function, i.e., the criteria used to select the next intermediate state (Subsection~\ref{subsec:abl-gf}),
    \item the pattern $\pattern_g$ used to reach the last computed intermediate state $P$ (Subsection~\ref{subsec:abl-patterng}),
    \item the pattern $\pattern_h$ used to progress beyond the last computed intermediate state $P$ (Subsection~\ref{subsec:abl-patternh}).
\end{enumerate}
Thus, we will consider various versions of \pattyd, each characterized by a triple representing the goal value function, the pattern $\pattern_g$ and the pattern $\pattern_h$ used. In particular:
\begin{enumerate}
    \item $\pattyc$ and $\pattyb$ use a goal value function with a {\em \underline{n}umeric} condition, the computed {\em \underline{p}lan} to reach $P$ from the starting state $S$ as $\pattern_g$, and a {\em \underline{c}omplete} pattern $\pattern_h$ to progress beyond $P$. They will therefore also be referred to as $\pattyc^{npc}$ and $\pattyb^{npc}$, respectively.
    \item $\pattyr$ (resp. $\pattyg$) uses a goal value function with a {\em \underline{n}umeric} condition, the {\em \underline{e}mpty} pattern to reach $P$ from the starting state $S$, and a {\em \underline{c}omplete} (resp. {\em \underline{i}ncomplete}) pattern $\pattern_h$ to progress beyond $P$. They will therefore also be referred to as $\pattyr^{nec}$ and $\pattyg^{nei}$, respectively.
\end{enumerate}

\subsection{Ablation Studies on the Goal Value Function}\label{subsec:abl-gf}

            \begin{table*}[tb]
            \centering
            \resizebox{\textwidth}{!}{
\begin{tabular}{|l||ccccccccc||ccccccccc||}
\hline
 & \multicolumn{9}{c||}{Solved (out of $20$)}&\multicolumn{9}{c||}{Time (s)}\\
Domain & $\mathrm{P}_\textsc{c}^{npc}$&$\mathrm{P}_\textsc{b}^{npc}$&$\mathrm{P}_\textsc{r}^{nec}$&$\mathrm{P}_\textsc{g}^{nei}$&$\mathrm{P}_\textsc{c}^{gpc}$&$\mathrm{P}_\textsc{b}^{gpc}$&$\mathrm{P}_\textsc{r}^{gec}$&$\mathrm{P}_\textsc{g}^{gei}$&$\mathrm{P}_\textsc{c}^{aes}$&$\mathrm{P}_\textsc{c}^{npc}$&$\mathrm{P}_\textsc{b}^{npc}$&$\mathrm{P}_\textsc{r}^{nec}$&$\mathrm{P}_\textsc{g}^{nei}$&$\mathrm{P}_\textsc{c}^{gpc}$&$\mathrm{P}_\textsc{b}^{gpc}$&$\mathrm{P}_\textsc{r}^{gec}$&$\mathrm{P}_\textsc{g}^{gei}$&$\mathrm{P}_\textsc{c}^{aes}$\\
\hline
\textsc{BlGrp} (S)&\textbf{20}&\textbf{20}&\textbf{20}&\textbf{20}&\textbf{20}&\textbf{20}&\textbf{20}&\textbf{20}&\textbf{20}&1.9&2.0&2.0&1.7&1.8&2.0&1.8&1.7&\textbf{1.6}\\
\textsc{Cnt} (S)&\textbf{20}&\textbf{20}&\textbf{20}&\textbf{20}&\textbf{20}&\textbf{20}&\textbf{20}&\textbf{20}&\textbf{20}&\textbf{0.8}&0.9&0.9&1.1&0.9&\textbf{0.8}&0.9&1.2&0.9\\
\textsc{Cnt} (L)&\textbf{20}&\textbf{20}&\textbf{20}&\textbf{20}&\textbf{20}&\textbf{20}&\textbf{20}&\textbf{20}&\textbf{20}&\textbf{0.9}&\textbf{0.9}&\textbf{0.9}&1.0&\textbf{0.9}&\textbf{0.9}&\textbf{0.9}&1.0&\textbf{0.9}\\
\textsc{Del} (S)&8&8&8&\textbf{9}&8&8&8&8&6&187.0&189.5&189.0&\textbf{171.8}&189.9&189.4&190.3&182.8&228.4\\
\textsc{Drn} (S)&\textbf{16}&\textbf{16}&13&\textbf{16}&\textbf{16}&\textbf{16}&13&\textbf{16}&3&110.8&99.9&110.8&98.2&100.4&95.4&131.9&\textbf{93.0}&255.3\\
\textsc{Exp} (S)&\textbf{4}&\textbf{4}&2&2&\textbf{4}&\textbf{4}&2&2&\textbf{4}&246.1&249.2&270.8&270.8&250.0&248.9&270.8&269.2&\textbf{244.5}\\
\textsc{Farm} (S)&\textbf{20}&\textbf{20}&\textbf{20}&17&\textbf{20}&\textbf{20}&\textbf{20}&17&\textbf{20}&\textbf{0.8}&0.9&\textbf{0.8}&45.7&\textbf{0.8}&0.9&\textbf{0.8}&45.7&\textbf{0.8}\\
\textsc{Farm} (L)&\textbf{20}&\textbf{20}&\textbf{20}&17&\textbf{20}&\textbf{20}&\textbf{20}&17&\textbf{20}&3.3&3.6&\textbf{3.2}&48.0&3.3&3.5&3.4&48.9&4.0\\
\textsc{HPwr} (S)&\textbf{20}&\textbf{20}&\textbf{20}&14&\textbf{20}&\textbf{20}&\textbf{20}&\textbf{20}&\textbf{20}&97.1&90.1&93.4&230.2&15.0&\textbf{14.0}&14.1&89.8&14.9\\
\textsc{Mrkt} (L)&2&5&5&\textbf{19}&-&-&-&-&-&294.8&279.2&279.8&\textbf{19.5}&-&-&-&-&-\\
\textsc{MPrime} (S)&\textbf{14}&\textbf{14}&\textbf{14}&10&\textbf{14}&\textbf{14}&\textbf{14}&10&12&\textbf{119.0}&128.7&124.3&154.9&120.5&124.3&124.0&154.7&131.9\\
\textsc{PathM} (S)&\textbf{20}&\textbf{20}&\textbf{20}&\textbf{20}&\textbf{20}&\textbf{20}&\textbf{20}&\textbf{20}&\textbf{20}&4.4&4.0&4.7&4.8&\textbf{3.7}&4.1&4.2&5.3&4.0\\
\textsc{PlWat} (S)&\textbf{20}&\textbf{20}&19&\textbf{20}&\textbf{20}&\textbf{20}&\textbf{20}&\textbf{20}&6&20.8&16.9&23.9&10.7&17.6&14.9&\textbf{9.2}&10.1&226.8\\
\textsc{Rvr} (S)&13&\textbf{14}&10&7&12&13&10&6&13&114.9&\textbf{113.6}&158.8&200.7&132.1&122.1&160.7&219.0&127.9\\
\textsc{Sail} (S)&\textbf{20}&\textbf{20}&\textbf{20}&\textbf{20}&\textbf{20}&\textbf{20}&\textbf{20}&\textbf{20}&\textbf{20}&1.1&\textbf{0.9}&1.0&\textbf{0.9}&1.0&1.0&1.0&1.0&1.4\\
\textsc{Sail} (L)&\textbf{20}&\textbf{20}&\textbf{20}&\textbf{20}&\textbf{20}&\textbf{20}&\textbf{20}&\textbf{20}&19&1.4&\textbf{0.8}&0.9&0.9&1.2&\textbf{0.8}&\textbf{0.8}&\textbf{0.8}&20.3\\
\textsc{Stlrs} (S)&\textbf{15}&\textbf{15}&\textbf{15}&\textbf{15}&\textbf{15}&\textbf{15}&\textbf{15}&\textbf{15}&\textbf{15}&83.5&83.7&82.8&\textbf{79.1}&82.3&82.7&83.1&\textbf{79.1}&88.5\\
\textsc{Sgr} (S)&\textbf{20}&\textbf{20}&15&15&\textbf{20}&\textbf{20}&17&16&\textbf{20}&12.2&14.9&81.3&78.4&9.8&13.2&56.6&68.8&\textbf{7.5}\\
\textsc{Tpp} (L)&3&3&\textbf{4}&3&2&2&\textbf{4}&3&2&259.3&255.6&\textbf{252.5}&264.6&270.2&270.1&256.1&264.8&274.6\\
\textsc{Zeno} (S)&\textbf{11}&\textbf{11}&\textbf{11}&\textbf{11}&\textbf{11}&\textbf{11}&\textbf{11}&\textbf{11}&\textbf{11}&136.8&\textbf{136.7}&137.1&137.8&137.0&\textbf{136.7}&136.9&137.1&138.0\\
\textsc{Line} (L)&\textbf{20}&\textbf{20}&\textbf{20}&\textbf{20}&\textbf{20}&\textbf{20}&\textbf{20}&\textbf{20}&\textbf{20}&1.6&1.5&1.4&1.4&1.6&1.4&1.5&1.4&\textbf{1.1}
\\\hline
\textit{Best}&\textbf{326}&\textbf{330}&\textbf{316}&\textbf{315}&\textbf{322}&\textbf{323}&\textbf{314}&\textbf{301}&\textbf{291}&\textbf{18}&\textbf{15}&\textbf{21}&\textbf{70}&\textbf{35}&\textbf{44}&\textbf{51}&\textbf{72}&\textbf{60}\\\hline

        \end{tabular}}
        \caption{Impact of the goal value function on Cautious, Brave, Reckless and Greedy \pattyd. Planner names are abbreviated.}
        \label{tab:abl-gf}
        \end{table*}

In \pattyc/\pattyb/\pattyr/\pattyg the formula $\Pi^\pattern_{S,P}$ is satisfiable if and only if there is a next state $P'$ reachable from $S$ with $\pattern$ which either extends the set of goals satisfied by $P$ or --assuming there is a single numeric goal $(\psi \unrhd 0)$ in $\Pi$-- which is closer than $P$ to a goal state according to the goal value function~(\ref{eq:gf-psi}).
Such condition differs from the one used by \citet{DBLP:conf/kr/CardelliniG25}, where the next intermediate state $P'$ had to satisfy at least one more goal and was selected among the ones satisfying as many additional goals as possible. Thus, here we study
\begin{enumerate}
    \item what is the impact of introducing search, and
    \item what is the impact of introducing the goal value function~(\ref{eq:gf-psi}) when there is a single numeric goal.
\end{enumerate}
We thus considered the system
\begin{enumerate}
    \item $\pattyc^{aes}$ in which the goal value function returns 0 if all the goals are satisfied, and 1 otherwise. In $\pattyc^{aes}$, {\em \underline{a}ll} the goals have to be satisfied in $P'$, both $S$ and $P$ are equal to $I$ and thus $\pattern_g$ is {\em \underline{e}mpty}, and $\pattern_h$ is {\em \underline{s}tatically} computed at the beginning in the initial state.\footnote{$\pattyc^{aes}$ is equal to the system called \pattygkr by \citet{DBLP:conf/kr/CardelliniG25}, and can be seen as the version of $\pattyo$ in which the next state variables, but the one corresponding to the bound, are eliminated.}
    \item $\pattyc^{gpc}/\pattyb^{gpc}/\pattyr^{gec}/\pattyg^{gei}$, each obtained from $\pattyc^{npc}$/$\pattyb^{npc}$/ $\pattyr^{nec}$/$\pattyg^{nei}$ by considering the {\em \underline{g}oal-count} as the only goal value function. Clearly, $\pattyc^{gpc}$ differs from $\pattyc^{npc}$ only on problems with a single numeric goal, and similarly for the other considered planners.
\end{enumerate}
Table~\ref{tab:abl-gf} shows the results. Considering the overall performances, each system with the numeric goal condition performs better than the corresponding one with only the goal count condition, and all the system doing search perform better than $\pattyc^{aes}$. The largest difference in between a ``numeric" and the corresponding ``goal-count" version is between 
$\pattyg^{nei}$ and $\pattyg^{gei}$ on the \textsc{HPwr(S)} and \textsc{Mrkt(L)} domains. Such differences can be explained by the fact that in both domains the problems have a single numeric goal. However, while \textsc{HPwr(S)} problems can be solved with a single \smt call when using a complete pattern, \textsc{Mrkt(L)} problems require multiple calls, and \texttt{z3} gets stuck with complete patterns. Notice that $\pattyc^{aes}$ solves 291 problems, as $\pattyo$, whereas \citet{DBLP:conf/kr/CardelliniG25} showed close but different numbers: this can be attributed to the newer version of \texttt{z3} employed in this paper.

\subsection{Ablation Studies on the Pattern to Reach the Intermediate State}\label{subsec:abl-patterng}

\begin{algorithm}
\caption{\textsc{Pattern\_$g$} function for pattern refinement via action elimination. \\
Input: A initial state $I$, 
a pattern $\pattern=a_1;a_2;\ldots;a_k$, an assignment $\mu$ to the actions in $\pattern$, the state $P$ resulting from the execution of the plan $a_1^{\mu(a_1)};a_2^{\mu(a_2)};\ldots;a_k^{\mu(a_k)}$. \\
Output: a subpattern of $\pattern$ allowing to reach $P$ starting from $I$.}
\label{alg:pattern-simplification}
\begin{algorithmic}[1]
\Function{\textsc{Pattern\_$g$}
}{$I,P,\mu,\pattern$}
\State $a_1; \ldots; a_n \gets \textsc{getPlan}(\mu,\pattern)$ \label{alg:ps-pattern2plan}   
\State $S \gets I$
\For{($i \gets 1$ \textbf{to} $n$)}
    \State $\hat{a}_i \gets \tuple{\pre(a_i) \cup \set{\psi = S(\psi) \mid x \asseq \psi \in \eff(a_i), x \in V_n},\eff(a_i)} $  \label{alg:ps-ce2strips}              
    \State $S \gets \res(a_i,S)$
\EndFor
\State $removed \gets \emptyset$
\State $S \gets I$
\For{($i \gets 1$ \textbf{to} $n$)}
    \State $marks \gets \emptyset$
    \If{($i \notin removed$)}
        \State $marks \gets marks \cup \{i\}$
        \State $S' \gets S$
        \For{($j \gets i+1$ \textbf{to} $n$)}
            \If{($j \notin removed$)}
                \If{($\hat{a}_j$ is not executable in $S'$)}
                    \State $marks \gets marks \cup \{j\}$
                \Else
                    \State $S' \gets \res(a_j,S')$
                \EndIf
            \EndIf
        \EndFor
        \If{($S' = P$)}
            \State $removed \gets removed \cup marks$
        \Else
            \State $S \gets \res(a_i,S)$
        \EndIf
    \EndIf
\EndFor
\State $\pi \gets a_1; a_2; \ldots; a_n \mbox{ with } \set{a_i \mid i \in removed} \mbox{ removed}$
\State \Return $\textsc{getPattern}(\pi)$
\EndFunction
\end{algorithmic}
\end{algorithm}

In \pattyc/\pattyb the pattern $\pattern_g$ is is computed by $\stp(I, P, \mu_P, \pattern_P)$ and $\stpn(I, P, \mu_P, \pattern_P, n)$ at Lines~\ref{alg:d:g-pattern} and \ref{alg:d:g-n-pattern} of \pattyd, and is the pattern of the plan $\pi$ defined as in~(\ref{eq:plan-pi}), assuming $\pattern_P = a_1;a_2;\ldots;a_k$.
Any pattern allowing to reach $P$ from $I$ will do. However, as already said, longer patterns $(i)$ cover more plans and thus may reduce the number of \smt calls, but $(ii)$ increase the number of action variables, thus likely making each \smt call more difficult.
The reason behind the choice we adopted for \pattyc/\pattyb is that $(i)$ the pattern of the plan $\pi$ can be computed in linear time, $(ii)$  is a subpattern of $\pattern_P$, and $(iii)$ is significantly shorter than $\pattern_P$ in many cases, thereby not introducing a significant overhead and making simpler the next call to the \smt solver for solving $\Pi^{\pattern}_{S,P}$.

We now investigate what is the effect of using either an optimized $\pattern_g$ by exploiting a quadratic greedy action elimination algorithm derived from \cite{fink1992formalizing,DBLP:conf/aips/MedC22}, or the likely redundant $\pattern_P$ as $\pattern_g$. The need for further optimization of the pattern $\pattern_g$ of the plan $\pi$, comes from the observation that $\pattern_g$ may  be redundant, i.e., contain  actions which can be removed still obtaining a pattern  allowing to reach $P$ from $I$. For the motivating example, we may have that the model found in the first iteration corresponds to the plan
\begin{equation}
    \label{eq:G1-plan-redundant}
\begin{array}{c}
\pi = \travel(1,2);\travel(2,3);\ldots; \travel(m-1,m); \buy(m); \buy(m); \sell(m). 
\end{array}
\end{equation}
which is redundant (indeed, it is not necessary to buy and then sell an object in market $m$) and leads to a redundant pattern having $\sell(m)$ as last action.

Though computing an irredundant pattern may be a desirable property, it comes with an extra price, since it is well known that checking whether a pattern is irredundant is already co-\textsc{np}-hard in the classical setting with no numeric variables~(see, e.g.,~\cite{DBLP:conf/ijcai/BercherHM24}).\footnote{If there are no numeric variables in $\Pi$ then no action is eligible for rolling and the pattern correspond to a plan.}
An intermediate solution between the mentioned linear and co-\textsc{np}-hard methods for computing a subpattern of $\pattern_P$ allowing to reach $P$ from $I$, are the greedy action elimination procedures, first proposed in the classical setting by \citet{fink1992formalizing} and recently extended to handle conditional effects by \citet{DBLP:conf/aips/MedC22}.

The procedure that we adopt for removing redundant actions from $\pi$ --presented in Algorithm~\ref{alg:pattern-simplification}-- is the standard procedure used in the classical setting, once actions with (possibly state dependent) numeric effects are converted into corresponding actions with state independent numeric effects at Line~\ref{alg:ps-ce2strips}. Such a step correspond to greedily trying to remove an action $a$ from $\pi$ and all the subsequent actions which either $(i)$ become non executable, or $(ii)$ whose effects are affected by the removal of $a$.
Notice that the second case can happen only when actions with state-dependent effects are allowed, as in numeric planning and also in classical planning with conditional effects. The procedure adopted by \citet{DBLP:conf/aips/MedC22}, also in the presence of state dependent effects, removes a subsequent action only if it does not become executable because of the previously removed actions, thus leaving the ones which are still executable though they are affected in their effects. In our motivating example, if we assume that the plan $a_1;\ldots;a_n$ computed at Line~\ref{alg:ps-pattern2plan} of Algorithm~\ref{alg:pattern-simplification}
is the one in (\ref{eq:G1-plan-redundant}), our procedure will successfully remove the last $\buy(m);\sell(m)$ actions (thus returning the sequence (\ref{eq:G1-plan})), while this will not be the case when considering the standard procedure. The added step can be performed in linear time, and the overall time required by the procedure remains quadratic.

Given the above, we consider the following systems:
\begin{enumerate}
    \item $\pattyc^{noc}/\pattyb^{noc}$ obtained from $\pattyc^{npc}/\pattyb^{npc}$ by exploiting the presented greedy action elimination procedure for computing an {\em \underline{o}ptimized} $\pattern_g$ with a quadratic time computation overhead, and
    \item $\pattyc^{nrc}/\pattyb^{nrc}$ obtained from $\pattyc^{npc}/\pattyb^{npc}$ by keeping the {\em \underline{r}edundant} $\pattern_P$ as $\pattern_g$, without incurring in any computation overhead.\footnote{$\pattyc^{nrc}$ is equal to the $\pattyhkr$ system presented by \citet{DBLP:conf/kr/CardelliniG25}.}
\end{enumerate}
Clearly, it does not make sense to consider the reckless and greedy versions of \pattyd since, in such cases, $\pattern_g$ is empty.

            \begin{table*}[t]
            \centering
            \resizebox{\textwidth}{!}{
\begin{tabular}{|l||cccccc||cccccc||}
\hline
 & \multicolumn{6}{c||}{Solved (out of $20$)}&\multicolumn{6}{c||}{Time (s)}\\
Domain & $\mathrm{P}_\textsc{c}^{npc}$&$\mathrm{P}_\textsc{c}^{nrc}$&$\mathrm{P}_\textsc{c}^{noc}$&$\mathrm{P}_\textsc{b}^{npc}$&$\mathrm{P}_\textsc{b}^{nrc}$&$\mathrm{P}_\textsc{b}^{noc}$&$\mathrm{P}_\textsc{c}^{npc}$&$\mathrm{P}_\textsc{c}^{nrc}$&$\mathrm{P}_\textsc{c}^{noc}$&$\mathrm{P}_\textsc{b}^{npc}$&$\mathrm{P}_\textsc{b}^{nrc}$&$\mathrm{P}_\textsc{b}^{noc}$\\
\hline
\textsc{BlGrp} (S)&\textbf{20}&\textbf{20}&\textbf{20}&\textbf{20}&\textbf{20}&\textbf{20}&1.9&\textbf{1.8}&\textbf{1.8}&2.0&1.9&\textbf{1.8}\\
\textsc{Cnt} (S)&\textbf{20}&\textbf{20}&\textbf{20}&\textbf{20}&\textbf{20}&\textbf{20}&\textbf{0.8}&0.9&0.9&0.9&0.9&0.9\\
\textsc{Cnt} (L)&\textbf{20}&\textbf{20}&\textbf{20}&\textbf{20}&\textbf{20}&\textbf{20}&\textbf{0.9}&1.0&\textbf{0.9}&\textbf{0.9}&1.0&\textbf{0.9}\\
\textsc{Del} (S)&\textbf{8}&5&\textbf{8}&\textbf{8}&6&\textbf{8}&\textbf{187.0}&228.1&187.4&189.5&212.9&191.6\\
\textsc{Drn} (S)&\textbf{16}&3&12&\textbf{16}&3&13&110.8&255.4&168.6&\textbf{99.9}&255.4&147.8\\
\textsc{Exp} (S)&\textbf{4}&\textbf{4}&\textbf{4}&\textbf{4}&\textbf{4}&\textbf{4}&246.1&\textbf{243.6}&246.1&249.2&245.0&244.7\\
\textsc{Farm} (S)&\textbf{20}&\textbf{20}&\textbf{20}&\textbf{20}&\textbf{20}&\textbf{20}&\textbf{0.8}&0.9&\textbf{0.8}&0.9&\textbf{0.8}&0.9\\
\textsc{Farm} (L)&\textbf{20}&\textbf{20}&\textbf{20}&\textbf{20}&\textbf{20}&\textbf{20}&3.3&3.5&\textbf{3.2}&3.6&3.3&3.4\\
\textsc{HPwr} (S)&\textbf{20}&19&\textbf{20}&\textbf{20}&\textbf{20}&\textbf{20}&97.1&104.0&90.8&\textbf{90.1}&93.6&92.2\\
\textsc{Mrkt} (L)&2&-&1&\textbf{5}&3&\textbf{5}&294.8&-&299.9&279.2&279.4&\textbf{269.9}\\
\textsc{MPrime} (S)&\textbf{14}&\textbf{14}&\textbf{14}&\textbf{14}&13&\textbf{14}&\textbf{119.0}&119.5&124.5&128.7&130.3&125.7\\
\textsc{PathM} (S)&\textbf{20}&\textbf{20}&\textbf{20}&\textbf{20}&\textbf{20}&\textbf{20}&4.4&4.7&4.3&\textbf{4.0}&4.4&4.6\\
\textsc{PlWat} (S)&\textbf{20}&6&18&\textbf{20}&6&\textbf{20}&20.8&217.4&134.6&\textbf{16.9}&219.7&85.5\\
\textsc{Rvr} (S)&13&12&13&\textbf{14}&12&\textbf{14}&114.9&143.1&121.6&113.6&141.8&\textbf{109.0}\\
\textsc{Sail} (S)&\textbf{20}&\textbf{20}&17&\textbf{20}&\textbf{20}&\textbf{20}&1.1&1.5&70.8&\textbf{0.9}&\textbf{0.9}&1.0\\
\textsc{Sail} (L)&\textbf{20}&\textbf{20}&16&\textbf{20}&\textbf{20}&\textbf{20}&1.4&10.8&60.9&\textbf{0.8}&\textbf{0.8}&\textbf{0.8}\\
\textsc{Stlrs} (S)&\textbf{15}&\textbf{15}&\textbf{15}&\textbf{15}&\textbf{15}&\textbf{15}&83.5&\textbf{82.8}&83.4&83.7&83.6&83.7\\
\textsc{Sgr} (S)&\textbf{20}&19&\textbf{20}&\textbf{20}&19&\textbf{20}&\textbf{12.2}&26.5&16.3&14.9&24.0&18.2\\
\textsc{Tpp} (L)&\textbf{3}&2&2&\textbf{3}&\textbf{3}&\textbf{3}&259.3&270.6&270.1&\textbf{255.6}&\textbf{255.6}&\textbf{255.6}\\
\textsc{Zeno} (S)&\textbf{11}&\textbf{11}&\textbf{11}&\textbf{11}&\textbf{11}&\textbf{11}&136.8&137.4&\textbf{136.6}&136.7&137.0&137.1\\
\textsc{Line} (L)&\textbf{20}&\textbf{20}&\textbf{20}&\textbf{20}&\textbf{20}&\textbf{20}&1.6&\textbf{1.2}&6.8&1.5&\textbf{1.2}&5.4
\\\hline
\textit{Best}&\textbf{326}&\textbf{290}&\textbf{311}&\textbf{330}&\textbf{295}&\textbf{327}&\textbf{66}&\textbf{60}&\textbf{44}&\textbf{86}&\textbf{53}&\textbf{53}\\\hline

        \end{tabular}}
        \caption{Impact of the pattern $\pattern_g$ on Cautious and Brave \pattyd. Planner names are abbreviated.}
        \label{tab:abl-patterng}
        \end{table*}
        
The results are presented in Table~\ref{tab:abl-patterng}. As it can be seen, none of the alternative versions of $\pattyc^{npc}/\pattyb^{npc}$ can improve the performance of the corresponding standard planner. In the case of $\pattyc^{noc}/\pattyb^{noc}$, the degradation in performance is due to the not very significant reduction in the pattern and the quadratic time spent to compute the reduced pattern. In the case of $\pattyc^{nrc}/\pattyb^{nrc}$, the degradation in performance is due to the significant number of additional variables in the encoding, which cause a significant degradation in the solving time of the corresponding formula.

\subsection{Ablation Studies on the Pattern to Progress the Search Beyond the Intermediate State}\label{subsec:abl-patternh}

In \pattyc/\pattyb/\pattyr/\pattyg the pattern $\pattern_h$ is dynamically recomputed in any intermediate state. Further,  while in \pattyc/\pattyb/\pattyr, the pattern $\pattern_h$ is complete, this is not the case for \pattyg, which can be seen as the version of $\pattyr$ when using an incomplete pattern.

We now investigate what is the effect of dynamically recomputing the pattern $\pattern_h$ in any intermediate state and what is the effect of using an incomplete $\pattern_h$ in the cautious and brave versions of $\pattyd$. We therefore consider the following systems:
\begin{enumerate}
    \item $\pattyc^{nps}/\pattyb^{nps}/\pattyr^{nps}/\pattyg^{nps}$ in which $\pattern_h$ is {\em \underline{s}tatic} and, more precisely, is the one computed in the initial state $I$, and
    \item 
    $\pattyc^{npi}/\pattyb^{npi}$ in which $\pattern_h$ is possibly {\em \underline{i}ncomplete} and computed as in $\pattyg^{nei}$.
\end{enumerate}

            \begin{table*}[tb]
            \centering
            \resizebox{\textwidth}{!}{
\begin{tabular}{|l||cccccccccc||cccccccccc||}
\hline
 & \multicolumn{10}{c||}{Solved (out of $20$)}&\multicolumn{10}{c||}{Time (s)}\\
Domain & $\mathrm{P}_\textsc{c}^{npc}$&$\mathrm{P}_\textsc{b}^{npc}$&$\mathrm{P}_\textsc{r}^{nec}$&$\mathrm{P}_\textsc{g}^{nei}$&$\mathrm{P}_\textsc{c}^{nps}$&$\mathrm{P}_\textsc{b}^{nps}$&$\mathrm{P}_\textsc{r}^{nes}$&$\mathrm{P}_\textsc{g}^{nes}$&$\mathrm{P}_\textsc{c}^{npi}$&$\mathrm{P}_\textsc{b}^{npi}$&$\mathrm{P}_\textsc{c}^{npc}$&$\mathrm{P}_\textsc{b}^{npc}$&$\mathrm{P}_\textsc{r}^{nec}$&$\mathrm{P}_\textsc{g}^{nei}$&$\mathrm{P}_\textsc{c}^{nps}$&$\mathrm{P}_\textsc{b}^{nps}$&$\mathrm{P}_\textsc{r}^{nes}$&$\mathrm{P}_\textsc{g}^{nes}$&$\mathrm{P}_\textsc{c}^{npi}$&$\mathrm{P}_\textsc{b}^{npi}$\\
\hline
\textsc{BlGrp} (S)&\textbf{20}&\textbf{20}&\textbf{20}&\textbf{20}&\textbf{20}&\textbf{20}&\textbf{20}&\textbf{20}&\textbf{20}&\textbf{20}&1.9&2.0&2.0&\textbf{1.7}&1.9&1.8&1.8&1.8&1.8&1.9\\
\textsc{Cnt} (S)&\textbf{20}&\textbf{20}&\textbf{20}&\textbf{20}&\textbf{20}&\textbf{20}&\textbf{20}&\textbf{20}&\textbf{20}&\textbf{20}&\textbf{0.8}&0.9&0.9&1.1&0.9&\textbf{0.8}&1.0&\textbf{0.8}&1.3&1.3\\
\textsc{Cnt} (L)&\textbf{20}&\textbf{20}&\textbf{20}&\textbf{20}&\textbf{20}&\textbf{20}&\textbf{20}&\textbf{20}&\textbf{20}&\textbf{20}&\textbf{0.9}&\textbf{0.9}&\textbf{0.9}&1.0&\textbf{0.9}&\textbf{0.9}&\textbf{0.9}&\textbf{0.9}&\textbf{0.9}&1.0\\
\textsc{Del} (S)&8&8&8&\textbf{9}&6&6&7&7&8&\textbf{9}&187.0&189.5&189.0&\textbf{171.8}&227.4&226.2&209.6&211.0&186.8&173.0\\
\textsc{Drn} (S)&\textbf{16}&\textbf{16}&13&\textbf{16}&\textbf{16}&\textbf{16}&13&13&\textbf{16}&\textbf{16}&110.8&99.9&110.8&98.2&107.6&103.6&125.2&141.4&94.7&\textbf{88.5}\\
\textsc{Exp} (S)&\textbf{4}&\textbf{4}&2&2&\textbf{4}&\textbf{4}&2&2&\textbf{4}&\textbf{4}&246.1&249.2&270.8&270.8&247.7&250.4&270.9&270.9&\textbf{244.6}&\textbf{244.6}\\
\textsc{Farm} (S)&\textbf{20}&\textbf{20}&\textbf{20}&17&\textbf{20}&\textbf{20}&\textbf{20}&\textbf{20}&\textbf{20}&\textbf{20}&\textbf{0.8}&0.9&\textbf{0.8}&45.7&0.9&\textbf{0.8}&0.9&\textbf{0.8}&1.0&1.0\\
\textsc{Farm} (L)&\textbf{20}&\textbf{20}&\textbf{20}&17&\textbf{20}&\textbf{20}&\textbf{20}&\textbf{20}&\textbf{20}&\textbf{20}&3.3&3.6&3.2&48.0&3.2&3.4&3.4&3.1&\textbf{0.8}&\textbf{0.8}\\
\textsc{HPwr} (S)&\textbf{20}&\textbf{20}&\textbf{20}&14&\textbf{20}&\textbf{20}&\textbf{20}&\textbf{20}&15&15&97.1&90.1&93.4&230.2&103.5&97.1&95.7&\textbf{88.6}&216.3&221.7\\
\textsc{Mrkt} (L)&2&5&5&\textbf{19}&-&3&4&3&8&8&294.8&279.2&279.8&\textbf{19.5}&-&281.1&277.8&278.9&200.1&190.5\\
\textsc{MPrime} (S)&\textbf{14}&\textbf{14}&\textbf{14}&10&\textbf{14}&\textbf{14}&\textbf{14}&\textbf{14}&10&10&\textbf{119.0}&128.7&124.3&154.9&126.5&125.7&123.8&119.6&155.9&154.2\\
\textsc{PathM} (S)&\textbf{20}&\textbf{20}&\textbf{20}&\textbf{20}&\textbf{20}&\textbf{20}&\textbf{20}&\textbf{20}&\textbf{20}&\textbf{20}&4.4&\textbf{4.0}&4.7&4.8&4.5&4.2&\textbf{4.0}&4.4&4.7&4.7\\
\textsc{PlWat} (S)&\textbf{20}&\textbf{20}&19&\textbf{20}&\textbf{20}&\textbf{20}&\textbf{20}&\textbf{20}&\textbf{20}&\textbf{20}&20.8&16.9&23.9&\textbf{10.7}&24.4&21.1&12.1&11.5&24.2&19.8\\
\textsc{Rvr} (S)&13&14&10&7&\textbf{15}&14&9&9&13&12&114.9&113.6&158.8&200.7&\textbf{106.9}&136.4&177.0&177.4&136.1&137.1\\
\textsc{Sail} (S)&\textbf{20}&\textbf{20}&\textbf{20}&\textbf{20}&\textbf{20}&\textbf{20}&\textbf{20}&\textbf{20}&\textbf{20}&\textbf{20}&1.1&\textbf{0.9}&1.0&\textbf{0.9}&1.0&1.0&1.0&1.0&1.1&1.0\\
\textsc{Sail} (L)&\textbf{20}&\textbf{20}&\textbf{20}&\textbf{20}&\textbf{20}&\textbf{20}&\textbf{20}&\textbf{20}&\textbf{20}&\textbf{20}&1.4&\textbf{0.8}&0.9&0.9&1.0&\textbf{0.8}&0.9&0.9&0.9&\textbf{0.8}\\
\textsc{Stlrs} (S)&\textbf{15}&\textbf{15}&\textbf{15}&\textbf{15}&\textbf{15}&\textbf{15}&\textbf{15}&\textbf{15}&\textbf{15}&\textbf{15}&83.5&83.7&82.8&79.1&82.2&84.1&82.8&83.3&\textbf{78.8}&79.4\\
\textsc{Sgr} (S)&\textbf{20}&\textbf{20}&15&15&\textbf{20}&\textbf{20}&15&15&\textbf{20}&\textbf{20}&12.2&14.9&81.3&78.4&9.9&10.3&83.0&83.6&\textbf{7.5}&7.9\\
\textsc{Tpp} (L)&3&3&\textbf{4}&3&2&2&3&3&2&2&259.3&255.6&\textbf{252.5}&264.6&270.1&270.1&264.1&264.0&270.6&270.4\\
\textsc{Zeno} (S)&\textbf{11}&\textbf{11}&\textbf{11}&\textbf{11}&\textbf{11}&\textbf{11}&\textbf{11}&\textbf{11}&\textbf{11}&\textbf{11}&136.8&\textbf{136.7}&137.1&137.8&137.8&138.3&137.6&137.9&147.8&138.1\\
\textsc{Line} (L)&\textbf{20}&\textbf{20}&\textbf{20}&\textbf{20}&\textbf{20}&\textbf{20}&\textbf{20}&\textbf{20}&\textbf{20}&\textbf{20}&1.6&1.5&\textbf{1.4}&\textbf{1.4}&1.6&1.5&1.6&1.5&1.5&\textbf{1.4}
\\\hline
\textit{Best}&\textbf{326}&\textbf{330}&\textbf{316}&\textbf{315}&\textbf{323}&\textbf{325}&\textbf{313}&\textbf{312}&\textbf{322}&\textbf{322}&\textbf{24}&\textbf{29}&\textbf{35}&\textbf{80}&\textbf{25}&\textbf{26}&\textbf{26}&\textbf{34}&\textbf{37}&\textbf{62}\\\hline

        \end{tabular}}
        \caption{Impact of a static (resp. incomplete) pattern $\pattern_h$ computation on Cautious, Brave, Reckless, Greedy (resp. Cautious, Brave) \pattyd. Planner names are abbreviated.}
        \label{tab:abl-patternh-incomplete}
        \end{table*}
        
The results are reported in Table~\ref{tab:abl-patternh-incomplete}. The first observation is that adopting a static pattern does not have a dramatic overall impact on performance, leading to improvements in some domains while causing a worsening in others. The \textsc{Mrkt(L)} domain shows the possible significant impact on performances with $\pattyg$. Such result can be expected since when the agent is in a state $S$ different from the initial state, it is likely that $(i)$ the actions which are executable in $S$ are not included in the initially computed incomplete pattern, and $(ii)$ the encodings become unsolvable as soon as complete patterns are considered. Exploiting an incomplete pattern $\pattern_h$ on top of the cautious and brave approaches rarely improves performance: $\pattyc^{npi}$ does not solve more problems than $\pattyc^{npc}$ in any domain, and $\pattyb^{npi}$ does so only in \textsc{Del(S)}. Moreover, the time differences are marginal.

\section{Conclusion}

We introduced a novel symbolic search-based procedure for numeric planning.
We proved its correctness and stated conditions for its completeness. We implemented a cautious/brave/reckless/greedy version of the procedure and showed that they attain different performance on different domains and that they perform comparatively well against all the available numeric planners on the benchmarks of the 2023 \ipc Agile Track and we performed ablation studies. 

This work opens new  avenues for future research on leveraging search techniques in symbolic planning. Indeed, this is the first work where the choice of how to encode the planning problem as a logic formula is guided by search. As the different performance of \pattyc/\pattyb/\pattyr/\pattyg on different domains show, more research is needed, e.g., to alternate different strategies to come up with a single more robust procedure and to deal with  problems where there is only one propositional goal, for which the strategies we propose here are ineffective.
Despite this, we have shown that \pattyc/\pattyb/\pattyr/\pattyg achieve remarkably good performance, even when compared to \textsc{sota} search-based planners. We are currently working on how to exploit alternative, novel search heuristic research in our symbolic setting. Moreover, it could be interesting to explore how to modify the \smt solver to exploit planning-specific branching heuristics, as proposed by \citet{DBLP:conf/aaai/GiunchigliaMS98,DBLP:journals/ai/Rintanen12}.

\printbibliography

\pagebreak
\appendix

\section{Infographics of the various \pattyd versions}\label{appendix:infographic}
Please find the infographics of the cautious/brave/reckless/greedy approaches in Fig. \ref{fig:infographic:cautious}/\ref{fig:infographic:brave}/\ref{fig:infographic:reckless}/\ref{fig:infographic:greedy}, respectively.

\section{Reproducibility Checklist for JAIR}

Select the answers that apply to your research -- one per item. 

\subsection*{All articles:}

\begin{enumerate}
    \item All claims investigated in this work are clearly stated. 
    [yes]
    \item Clear explanations are given how the work reported substantiates the claims. 
    [yes]
    \item Limitations or technical assumptions are stated clearly and explicitly. 
    [yes]
    \item Conceptual outlines and/or pseudo-code descriptions of the AI methods introduced in this work are provided, and important implementation details are discussed. 
    [yes]
    \item 
    Motivation is provided for all design choices, including algorithms, implementation choices, parameters, data sets and experimental protocols beyond metrics.
    [yes]
\end{enumerate}

\subsection*{Articles containing theoretical contributions:}
Does this paper make theoretical contributions? 
[yes] 

If yes, please complete the list below.

\begin{enumerate}
    \item All assumptions and restrictions are stated clearly and formally. 
    [yes]
    \item All novel claims are stated formally (e.g., in theorem statements). 
    [yes]
    \item Proofs of all non-trivial claims are provided in sufficient detail to permit verification by readers with a reasonable degree of expertise (e.g., that expected from a PhD candidate in the same area of AI). [yes]
    \item
    Complex formalism, such as definitions or proofs, is motivated and explained clearly.
    [yes]
    \item 
    The use of mathematical notation and formalism serves the purpose of enhancing clarity and precision; gratuitous use of mathematical formalism (i.e., use that does not enhance clarity or precision) is avoided.
    [yes]
    \item 
    Appropriate citations are given for all non-trivial theoretical tools and techniques. 
    [yes]
\end{enumerate}

\subsection*{Articles reporting on computational experiments:}
Does this paper include computational experiments? [yes]

If yes, please complete the list below.
\begin{enumerate}
    \item 
    All source code required for conducting experiments is included in an online appendix 
    or will be made publicly available upon publication of the paper.
    The online appendix follows best practices for source code readability and documentation as well as for long-term accessibility.
    [yes]
    \item The source code comes with a license that
    allows free usage for reproducibility purposes.
    [yes]
    \item The source code comes with a license that
    allows free usage for research purposes in general.
    [yes]
    \item 
    Raw, unaggregated data from all experiments is included in an online appendix 
    or will be made publicly available upon publication of the paper.
    The online appendix follows best practices for long-term accessibility.
    [yes]
    \item The unaggregated data comes with a license that
    allows free usage for reproducibility purposes.
    [yes]
    \item The unaggregated data comes with a license that
    allows free usage for research purposes in general.
    [yes]
    \item If an algorithm depends on randomness, then the method used for generating random numbers and for setting seeds is described in a way sufficient to allow replication of results. 
    [NA]
    \item The execution environment for experiments, the computing infrastructure (hardware and software) used for running them, is described, including GPU/CPU makes and models; amount of memory (cache and RAM); make and version of operating system; names and versions of relevant software libraries and frameworks. 
    [yes]
    \item 
    The evaluation metrics used in experiments are clearly explained and their choice is explicitly motivated. 
    [yes]
    \item 
    The number of algorithm runs used to compute each result is reported. 
    [yes]
    \item 
    Reported results have not been ``cherry-picked'' by silently ignoring unsuccessful or unsatisfactory experiments. 
    [yes]
    \item 
    Analysis of results goes beyond single-dimensional summaries of performance (e.g., average, median) to include measures of variation, confidence, or other distributional information. 
    [yes]
    \item 
    All (hyper-) parameter settings for 
    the algorithms/methods used in experiments have been reported, along with the rationale or method for determining them. 
    [NA]
    \item 
    The number and range of (hyper-) parameter settings explored prior to conducting final experiments have been indicated, along with the effort spent on (hyper-) parameter optimisation. 
    [NA]
    \item 
    Appropriately chosen statistical hypothesis tests are used to establish statistical significance
    in the presence of noise effects.
    [NA]
\end{enumerate}

\subsection*{Articles using data sets:}
Does this work rely on one or more data sets (possibly obtained from a benchmark generator or similar software artifact)? 
[yes]

If yes, please complete the list below.
\begin{enumerate}
    \item 
    All newly introduced data sets 
    are included in an online appendix 
    or will be made publicly available upon publication of the paper.
    The online appendix follows best practices for long-term accessibility with a license
    that allows free usage for research purposes.
    [yes]
    \item The newly introduced data set comes with a license that
    allows free usage for reproducibility purposes.
    [yes]
    \item The newly introduced data set comes with a license that
    allows free usage for research purposes in general.
    [yes]
    \item All data sets drawn from the literature or other public sources (potentially including authors' own previously published work) are accompanied by appropriate citations.
    [yes]
    \item All data sets drawn from the existing literature (potentially including authors’ own previously published work) are publicly available. [yes]
    \item All new data sets and data sets that are not publicly available are described in detail, including relevant statistics, the data collection process and annotation process if relevant.
    [NA]
    \item 
    All methods used for preprocessing, augmenting, batching or splitting data sets (e.g., in the context of hold-out or cross-validation)
    are described in detail. [NA]
\end{enumerate}

\begin{figure}[p]
    \centering
    \includegraphics[width=0.9\linewidth]{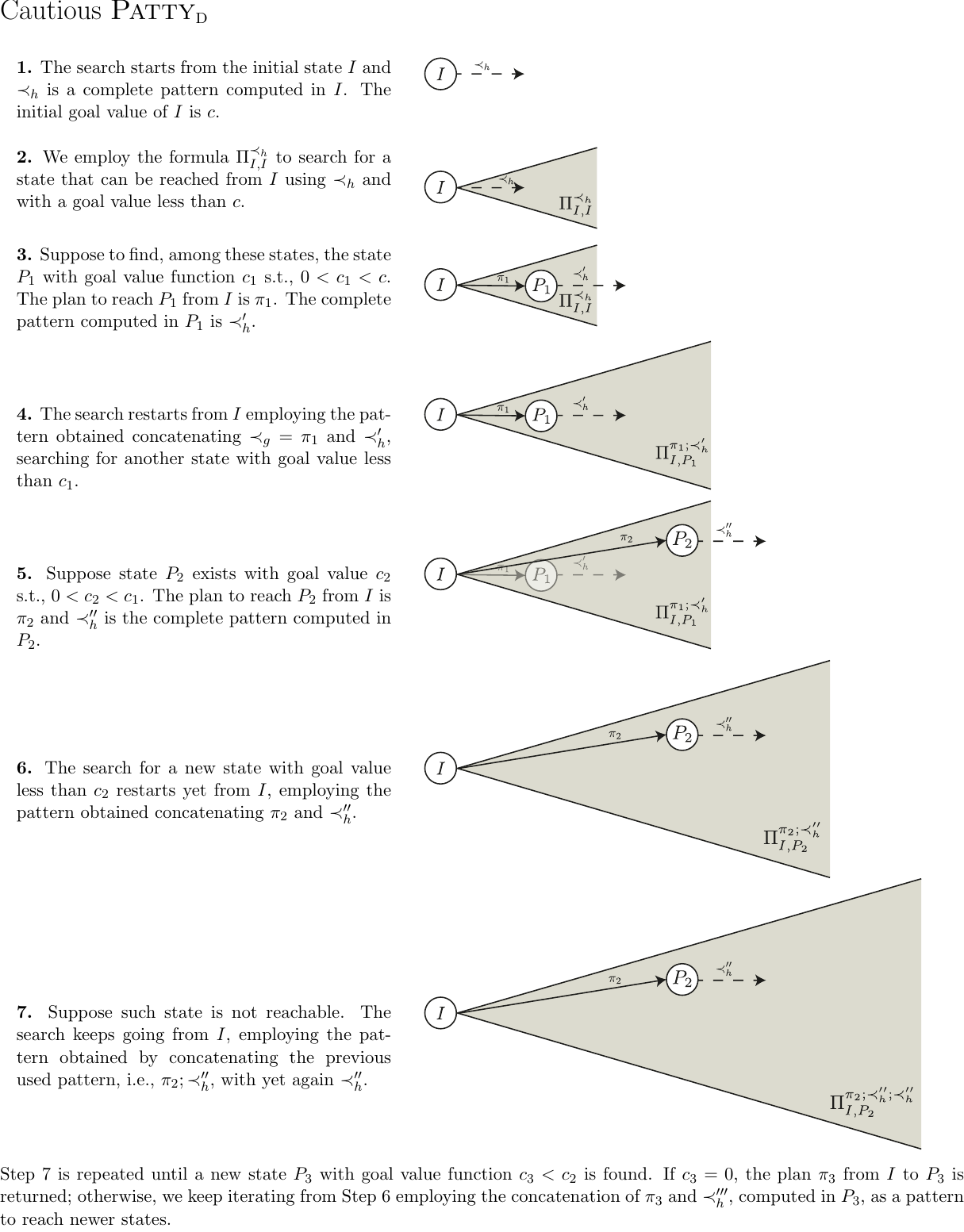}
    \caption{Infographic showing an example of the cautious \pattyd search.}
    \label{fig:infographic:cautious}
\end{figure}

\begin{figure}[p]
    \centering
    \includegraphics[width=0.9\linewidth]{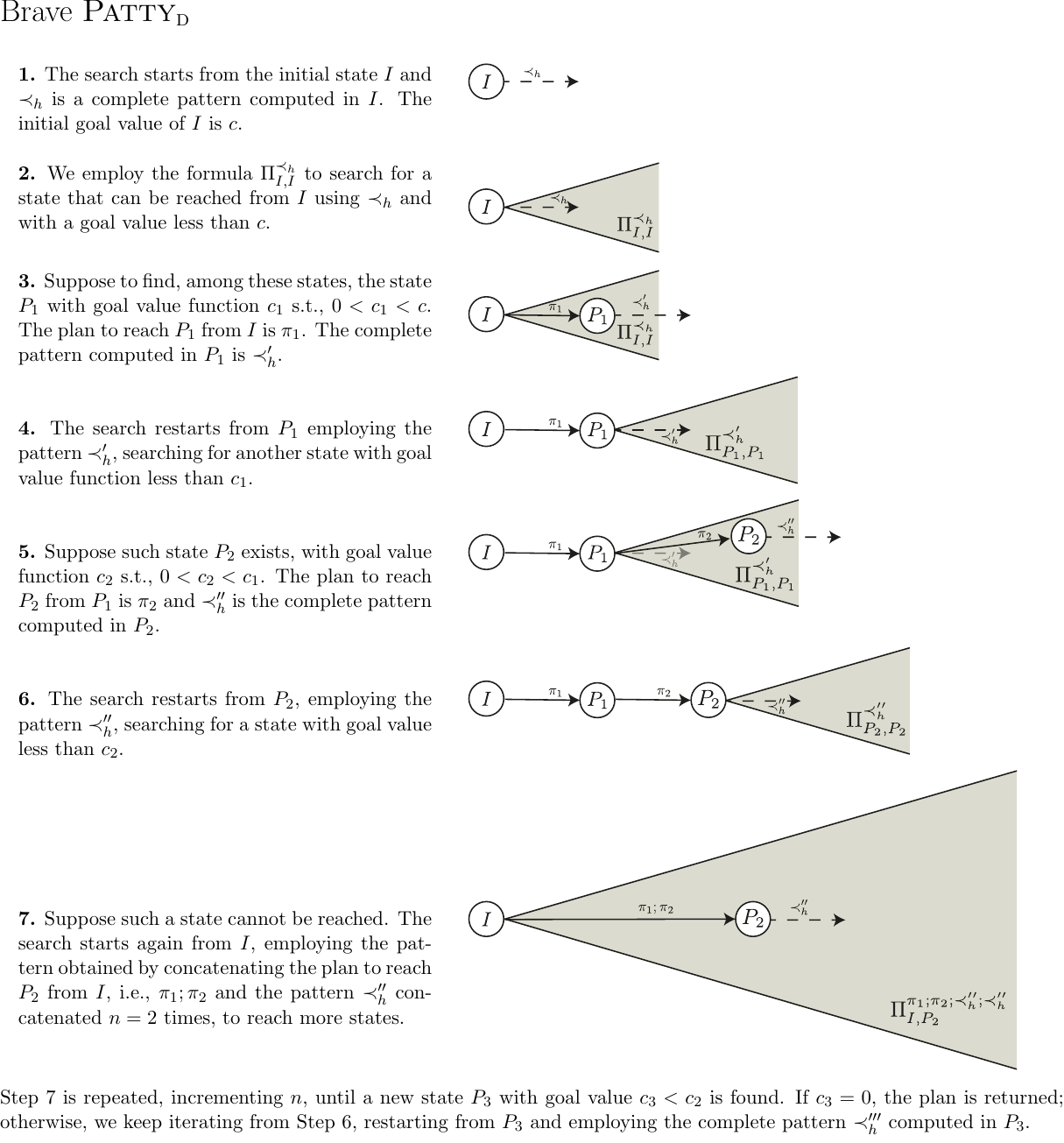}
    \caption{Infographic showing an example of the brave \pattyd search.}
    \label{fig:infographic:brave}
\end{figure}

\begin{figure}[p]
    \centering
    \includegraphics[width=0.9\linewidth]{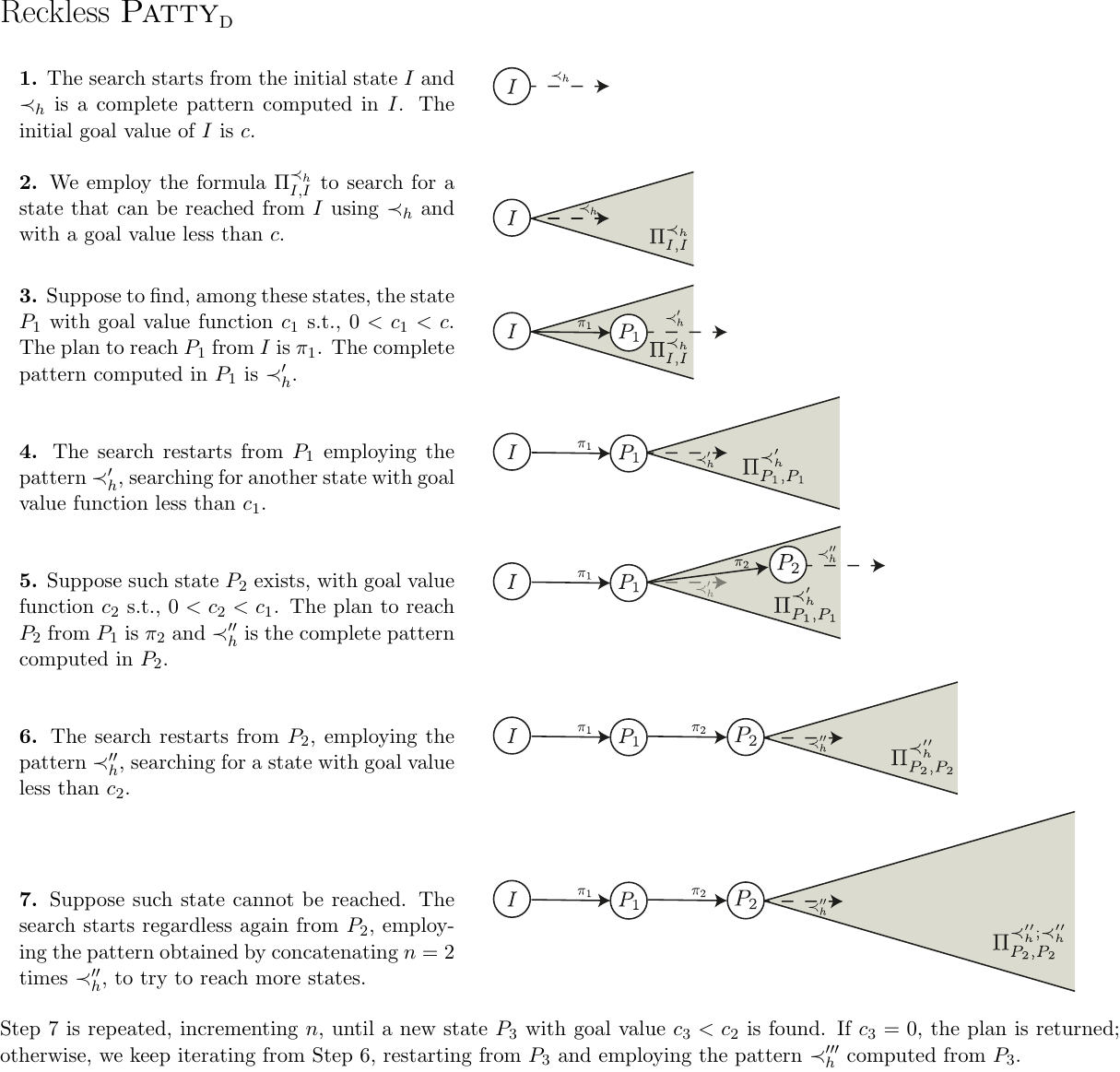}
    \caption{Infographic showing an example of the reckless \pattyd search.}
    \label{fig:infographic:reckless}
\end{figure}

\begin{figure}[p]
    \centering
    \includegraphics[width=0.9\linewidth]{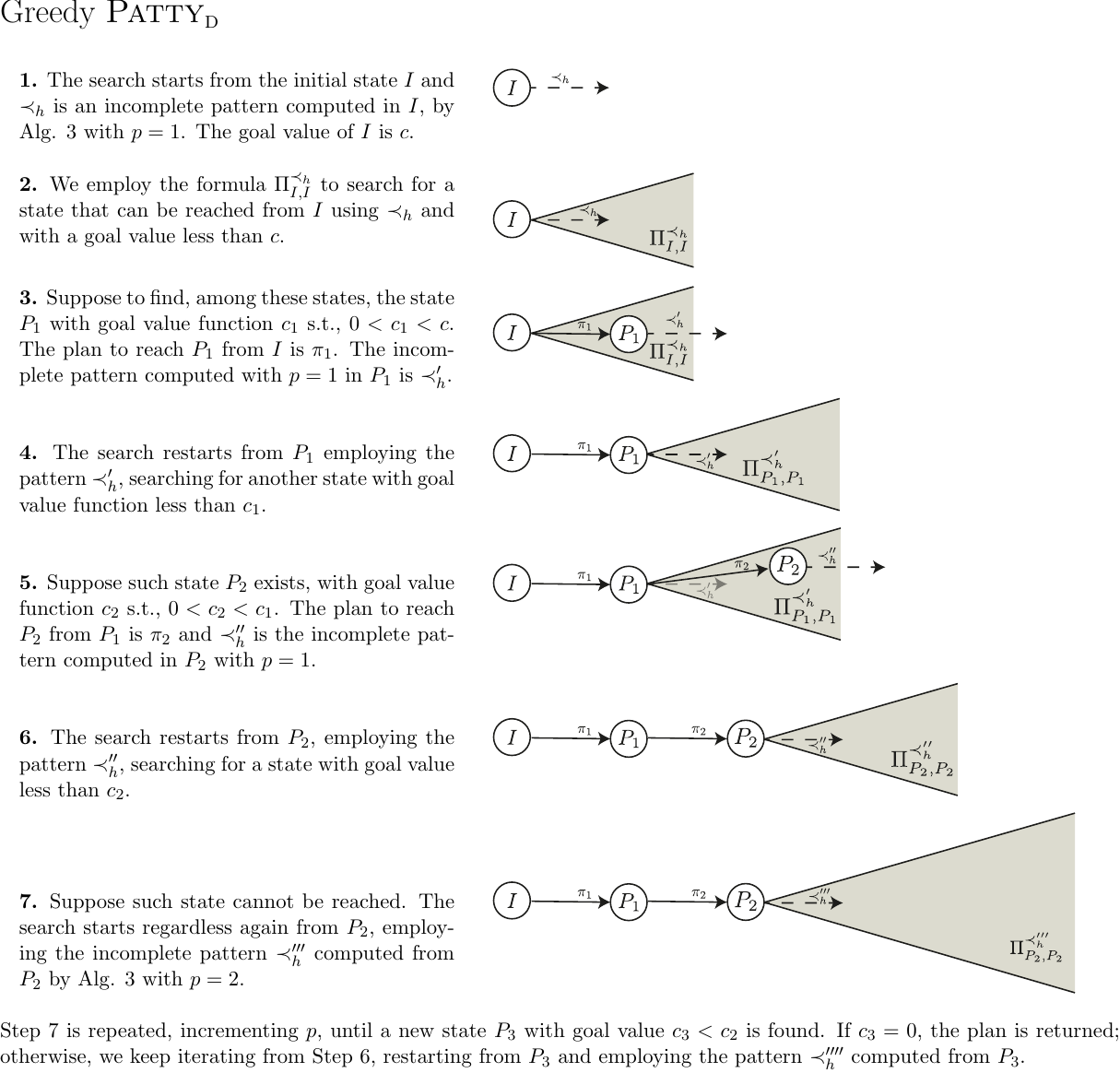}
    \caption{Infographic showing an example of the greedy \pattyd search.}
    \label{fig:infographic:greedy}
\end{figure}

\end{document}